\documentclass{article}

\usepackage[final, nonatbib]{neurips_2019}

\usepackage[utf8]{inputenc} 
\usepackage[T1]{fontenc}    
\usepackage{hyperref}       
\usepackage{url}            
\usepackage{booktabs}       
\usepackage{amsfonts}       
\usepackage{nicefrac}       
\usepackage{microtype}      
\usepackage{algorithm}
\usepackage{algorithmic}
\usepackage{graphicx}
\usepackage{subfigure}
\usepackage{caption}
\usepackage{pdfpages}
\usepackage{amsthm}
\usepackage{url}
\usepackage{comment}
\usepackage{amsmath}
\usepackage{amssymb}

\usepackage{listings}
\usepackage{xcolor}

\definecolor{codegreen}{rgb}{0,0.6,0}
\definecolor{codegray}{rgb}{0.5,0.5,0.5}
\definecolor{codepurple}{rgb}{0.58,0,0.82}
\definecolor{backcolour}{rgb}{0.95,0.95,0.92}

\lstdefinestyle{mystyle}{
    backgroundcolor=\color{backcolour},
    commentstyle=\color{codegreen},
    keywordstyle=\color{magenta},
    numberstyle=\tiny\color{codegray},
    stringstyle=\color{codepurple},
    basicstyle=\ttfamily\tiny,
    breakatwhitespace=false,
    breaklines=true,
    captionpos=b,
    keepspaces=true,
    numbers=left,
    numbersep=5pt,
    showspaces=false,
    showstringspaces=false,
    showtabs=false,
    tabsize=2
}

\usepackage{color}
\newcount\Comments  
\definecolor{darkgreen}{rgb}{0,0.5,0}
\definecolor{darkred}{rgb}{0.7,0,0}
\definecolor{teal}{rgb}{0.3,0.8,0.8}
\newcommand{\kibitz}[2]{\ifnum\Comments=1\textcolor{#1}{#2}\fi}

\title{Explicit Explore-Exploit Algorithms in Continuous State Spaces}

\author{
  Mikael Henaff \\
  Microsoft Research \\
  \texttt{mihenaff@microsoft.com}
}

\newtheorem{theorem}{Theorem}
\newtheorem{lemma}{Lemma}
\newtheorem{proposition}{Proposition}

\newtheorem{techlemma}{Technical Lemma}

\newtheorem{definition}{Definition}

\newcommand{\rulesep}{\unskip\ \vrule\ }

\begin{document}

\maketitle

\begin{abstract}
  We present a new model-based algorithm for reinforcement learning (RL) which consists of explicit exploration and exploitation phases, and is applicable in large or infinite state spaces.
  The algorithm maintains a set of dynamics models consistent with current experience and explores by finding policies which induce high disagreement between their state predictions. It then exploits using the refined set of models or experience gathered during exploration.
  We show that under realizability and optimal planning assumptions, our algorithm provably finds a near-optimal policy with a number of samples that is polynomial in a structural complexity measure which we show to be low in several natural settings. We then give a practical approximation using neural networks and demonstrate its performance and sample efficiency in practice.
\end{abstract}

\section{Introduction}

What is a good algorithm for systematically exploring an environment for the purpose of reinforcement learning?
A good answer could make the application of deep RL to complex problems \cite{mnih2015humanlevel, mnih16, Lillicrap2016ContinuousCW, Hessel2018RainbowCI} much more sample efficient.
In tabular Markov Decision Processes (MDPs) with a small number of discrete states, model-based algorithms which perform exploration in a provably sample-efficient manner have existed for over a decade \cite{Kearns2002, Brafman2003, Strehl2005}.
The first of these, known as the Explicit Explore-Exploit ($E^3$) algorithm \cite{Kearns2002}, progressively builds a model of the environment's dynamics.
At each step, the agent uses this model to plan, either to \textit{explore} and reach an unknown state, or to \textit{exploit} and maximize its reward within the states it knows well.
By actively seeking out unknown states, the algorithm provably learns a near-optimal policy using a number of samples which is at most polynomial in the number of states.
Many problems of interest, however, have a set of states which is infinite or extremely large (for example, all images represented with finite precision), and in these settings, tabular algorithms are no longer applicable.

In this work, we propose a new $E^3$-style algorithm which operates in large or continuous state spaces.
The algorithm maintains a set of dynamics models which are consistent with the agent's current experience, and explores the environment by executing policies designed to induce high disagreement between their predictions.
We show that under realizability and optimal planning assumptions, our algorithm provably finds a near-optimal policy using a number of samples from the environment which is independent of the number of states, and is instead polynomial in the rank of the \textit{model misfit matrix}, a structural complexity measure which we show to be low in natural settings such as small tabular MDPs, large MDPs with factored transition dynamics \cite{Kearns1999} and (potentially infinite) low rank MDPs.
We then present a practical version of the algorithm using neural networks, and demonstrate its performance and sample efficiency empirically on several problems with large or continuous state spaces.


\section{Algorithm}

We consider an episodic, finite-horizon MDP setting defined by a tuple $(\mathcal{S}, \mathcal{A}, M^\star, R^\star, H)$. Here $\mathcal{S}$ is a set of states (which could be large or infinite), $\mathcal{A}$ is a discrete set of actions, $M^\star$ is the true (unknown) transition model mapping state-action pairs to distributions over next states, $R^\star$ is the true function mapping states to rewards in $[0, 1]$, and $H$ is the horizon length.
For simplicity we assume rewards are part of the state and the agent has access to $R^\star$, so the task of predicting future rewards is included in that of predicting future states.
A state $s \in \mathcal{S}$ at time step $h$ will be denoted by $s_h$.

\begin{algorithm}[t]
  \begin{algorithmic}[1]
    \STATE \textbf{Inputs} Initial model set $\mathcal{M}$, policy class $\Pi$, number of trajectories $n$, tolerance $\epsilon$, model error $\phi$.
    \STATE $\mathcal{M}_1 \leftarrow \mathcal{M}$
    \STATE Initialize replay buffer $\mathcal{R} \leftarrow \emptyset$.
    \FOR{$t=1, 2, ...$}
    \STATE $\pi_\mathrm{explore}^t = \mbox{argmax}_{\pi \in \Pi} \Big[ v_{\mathrm{explore}}(\pi, \mathcal{M}_{t}) \Big]$
    \IF{$v_{\mathrm{explore}}(\pi_{\mathrm{explore}}^t, \mathcal{M}_{t}) > \frac{\epsilon}{|\mathcal{A}|}$}
    \STATE Collect dataset of $n$ trajectories following $\pi_{\mathrm{explore}}^t$, add to replay buffer $\mathcal{R}$
    \STATE $\mathcal{M}_{t+1} \leftarrow \texttt{UpdateModelSet}(\mathcal{M}_{t}, \mathcal{R}, \phi)$
    \ELSE
    \STATE Choose any $\tilde{M} \in \mathcal{M}_{t}$
    \STATE $\pi_{\mathrm{exploit}} = \mbox{argmax}_{\pi \in \Pi} \Big[ v_{\mathrm{exploit}}(\pi, \tilde{M}) \Big]$
    \STATE Halt and return $\pi_{\mathrm{exploit}}$
    \ENDIF
    \ENDFOR
\end{algorithmic}
\caption{$(\mathcal{M}, \Pi, n, \epsilon, \phi)$}
\label{alg:main}
\end{algorithm}

The general form of our algorithm is given by Algorithm \ref{alg:main}.
At each epoch $t$, the algorithm maintains a set of dynamics models $\mathcal{M}_{t}$ which are consistent with the experience accumulated so far,
and searches for an exploration policy which will induce high disagreement between their predictions.
If such a policy is found, it is executed and the set of models is updated to reflect the new experience.
Otherwise, the algorithm switches to its exploit phase and searches for a policy which will maximize its predicted future rewards.

Let $P^{\pi, h}_M(\cdot)$ denote the distribution over states at time step $h$ induced by sampling actions from policy $\pi$ and transitions from model $M$, and let $\mathcal{D}(\pi, M, M', h) = \delta(P^{\pi, h}_M(\cdot), P^{\pi, h}_{M'}(\cdot))$, where $\delta$ denotes a distance measure between probability distributions such as KL divergence or total variation.
The quantity which the exploration policy seeks to maximize at epoch $t$ is given by:

\begin{align*}
  v_{\mathrm{explore}}(\pi, \mathcal{M}_{t}) = \max_{M, M' \in \mathcal{M}_{t}} \sum_{h=1}^H \mathcal{D}(\pi, M, M', h)
\end{align*}

Maximizing this quantity can be viewed as solving a fictitious \textit{exploration MDP}, whose state space is the concatenation of $|\mathcal{M}_{t}|$ state vectors in the original MDP, whose transition matrix consists of a block-diagonal matrix whose blocks are the transition matrices of the models in $\mathcal{M}_{t}$, and whose reward function is the distance measured using $\delta$ between the pairs of components of the state vector corresponding to different models.
Importantly, searching for an exploration policy can be done internally by the agent and does not require any environment interaction, which will be key to the algorithm's sample efficiency.

Once the agent can no longer find a policy which induces sufficient disagreement between its candidate models in $\mathcal{M}_t$, it chooses a model and computes an exploitation policy using the model's predicted reward:

\begin{align*}
  v_{\mathrm{exploit}}(\pi, M) = \sum_{h=1}^H \sum_{s_h}P^{\pi, h}_M(s_h) R^\star(s_h)
\end{align*}

\section{Sample Complexity Analysis}

\subsection{Algorithm Instantiation}

We first give an instantiation of Algorithm \ref{alg:main}, called DREEM (DisagReement-led Elimination of Environment Models), for which we will prove sample complexity results. All proofs can be found in Appendix \ref{appendix-proofs}.
The algorithm starts with a large set of candidate models $\mathcal{M}$, which is assumed to contain the true model, and iteratively eliminates models which are not consistent with the experience gathered through the exploration policy.
We will show that the number of samples needed to find a near-optimal policy is independent of the number of states, and is instead polynomial in $|\mathcal{A}|, H, \log{|\mathcal{M}|}, \log{|\Pi|}$, and the rank of the $\textit{model misfit matrix}$, a quantity which we define below and which is low in natural settings.
DREEM is identical to Algorithm \ref{alg:main}, with the \texttt{UpdateModelSet} subroutine instantiated as follows:

\begin{algorithm}[ht!]
  \begin{algorithmic}[1]
    \STATE For each $M \in \mathcal{M}_t, h \leq H$, compute $\mathcal{\widetilde{W}}(\pi_\mathrm{explore}^t, M, h)$ using data from $\mathcal{R}$ collected using the last exploration policy $\pi_\mathrm{explore}^t$
    \STATE $\mathcal{M}_{t+1} \leftarrow \{M \in \mathcal{M}_{t}: \widetilde{\mathcal{W}}(\pi_\mathrm{explore}^t, M, h) \leq \phi \mbox{ for all } h \leq H\}$
    \STATE Return $\mathcal{M}_{t+1}$
\end{algorithmic}
\caption{$\texttt{UpdateModelSet}(\mathcal{M}_{t}, \mathcal{R}, \phi)$}
\label{alg:update-version-space}
\end{algorithm}

The quantity $\mathcal{W}(\pi, M, h)$ can be thought of as the error of model $M$ in parts of the state space visited by $\pi$ at time step $h$, and is formally defined below.

\begin{definition}
  The misfit of model $M$ discovered by policy $\pi$ at time step $h$ is given by:
    \begin{equation*}
      \mathcal{W}(\pi, M, h) = \mathbb{E}_{s_{h-1} \sim P_{M^\star}^{\pi, h-1}, a_{h-1} \sim U(\mathcal{A})}\big[ \|P_M(\cdot | s_{h-1}, a_{h-1}) - P_{M^\star}(\cdot | s_{h-1}, a_{h-1}) \|_{TV} \big]
    \end{equation*}

    The empirical misfit estimated using a dataset collected by following $\pi$ is denoted $\widetilde{\mathcal{W}}(\pi, M, h)$.
\end{definition}

See Appendix \ref{empirical-misfit} for details on computing $\widetilde{\mathcal{W}}$.
We will make use of the following two assumptions:

\textbf{Assumption 1.} $\mathcal{M}$ contains the true model $M^\star$ and $\Pi$ contains optimal policies for all models in $\mathcal{M}$.

\textbf{Assumption 2.} The policy optimizations in Algorithm \ref{alg:main} are performed exactly.

The first is a standard realizability assumption. The second assumes access to an optimal planner and has been used in several previous works \cite{Kearns2002, Kearns1999, Kakade2003, Brafman2003}. This does not mean that the planning problem is trivial, but is meant to separate the difficulty of planning from that of exploration.

We note that DREEM will not be computationally feasible for many problems since the sets $\mathcal{M}_t$ will often be large, and the algorithm requires iterating over them during both the elimination and planning steps. However, it distills the key ideas of Algorithm \ref{alg:main} and demonstrates its sample efficiency when optimizations can be performed exactly. We will later give a practical instantiation of Algorithm \ref{alg:main} and demonstrate its sample efficiency empirically.

\subsection{Structural Complexity Measure}
Since we are considering settings where $\mathcal{S}$ is large or infinite, it is not meaningful to give sample complexity results in terms of the number of states, as is often done for tabular algorithms.
We instead use a structural complexity measure which is independent of the number of states, and depends on the maximum rank over a set of error matrices, which we define next.
\footnote{We use a generalized notion of rank with a condition on the row norms of the factorization: for an $m \times n$ matrix $B$, denote $rank(B, \beta)$ to be the smallest integer $k$ such that $B=UV^\top$ with $U \in \mathbb{R}^{m \times k}, V \in \mathbb{R}^{n \times k}$ and for every pair of rows $u_i, v_j$ we have $\|u_i\|_2 \cdot \|v_j \|_2 \leq \beta$. $\beta$ appears in Lemma 3 and Theorem 1.}

\begin{definition}
  (Model Misfit Matrices)
  Let $\mathcal{M}$ be a model class and $\Pi$ a policy class.
  Define the set of matrices $A_1, ..., A_H \in \mathbb{R}^{|\Pi| \times |\mathcal{M}|}$ by $A_h(\pi, M) = \mathcal{W}(\pi, M, h)$ for all $\pi \in \Pi$ and $M \in \mathcal{M}$.
\end{definition}

Using the ranks of error matrices as complexity measures of RL environments was proposed in \cite{jiang2017, Sun2018}.
Although the model misfit matrices $A_h$ may themselves be very large, we show next that their ranks are in fact small in several natural settings.

\begin{proposition}
Assume $|\mathcal{S}|$ is finite and let $A_h$ be the matrix defined above. Then $rank(A_h) \leq |\mathcal{S}|$.
\end{proposition}

\begin{proposition}
  Let $\Gamma$ denote the true transition matrix of size $|\mathcal{S}| \times |\mathcal{S} \times \mathcal{A}|$, with $\Gamma(s', (s, a)) = P_{M^\star}(s'|s, a)$.
  Assume that there exist two matrices $\Gamma_1, \Gamma_2$ of sizes $|\mathcal{S}| \times K$ and $K \times |\mathcal{S} \times \mathcal{A}|$ such that $\Gamma=\Gamma_1\Gamma_2$.
  Then $rank(A_h) \leq K$.
\end{proposition}

The next proposition, which is a straightforward adaptation of a result from \cite{Sun2018} \footnote{Appendix E.2, Proposition 2}, shows that the ranks of the model misfit matrices are also low in factored MDPs \cite{Kearns1999}.

\begin{proposition}
  Consider a factored MDP setting where the state space is given by $\mathcal{S} = \mathcal{O}^d$ where $d \in \mathbb{N}$ and $\mathcal{O}$ is a small finite set, and the transition matrix has a factored structure with $L$ parameters. Then $rank(A_h) \leq L$.
\end{proposition}

\subsection{Sample Complexity}

Now that we have defined our structural complexity measure, we prove sample complexity results for DREEM. We will use a slightly different definition of $\mathcal{D}(\pi, M, M')$ than the one in Section 2, in that the last action is sampled uniformly:

\begin{definition}
  (Predicted Model Disagreement Induced by Policy)

  \begin{align*}
    \mathcal{D}(\pi, M, M', h) &= \sum_{s_{h-1}} \sum_{a_{h-1}}\sum_{s_h} \Big| P_M(s_h | s_{h-1}, a_{h-1}) P_M^{\pi, h-1}(s_{h-1}) U(a_h) \\
    & \phantom{=========} - P_{M'}(s_h | s_{h-1}, a_{h-1}) P_{M'}^{\pi, h-1}(s_{h-1}) U(a_h)  \Big| \\
  \end{align*}

\end{definition}

We begin by proving a lemma which, intuitively, states that if a policy induces disagreement between two models of the environment, then it will also induce disagreement between at least one of these models and the true model. This means that by searching for and then executing a policy which induces disagreement between at least two models, the agent will collect experience from the environment which will enable it to invalidate at least one of them.

\begin{lemma}
  \label{disagreement-error}
  Let $\mathcal{M}$ be a set of models and $\Pi$ a set of policies.
  If there exist $M, M' \in \mathcal{M}, \pi \in \Pi$ and $h \leq H$ such that $\mathcal{D}(\pi, M, M', h) > \alpha$, then there exists $h' \leq h$ such that $\mathcal{W}(\pi, M, h') > \frac{\alpha}{4|\mathcal{A}|\cdot H}$ or $\mathcal{W}(\pi, M', h') > \frac{\alpha}{4|\mathcal{A}|\cdot H}$ (or both).
\end{lemma}

Next, we give a lemma which states that at any time step, the agent has either found an exploration policy which will lead it to collect experience which allows it to reduce its set of candidate models, or has found an exploitation policy which is close to optimal. Here $v_\pi$ is the value of $\pi$ in the true MDP.
\begin{lemma}
  \label{terminate-explore}
  (Explore or Exploit)
  Suppose the true model $M^\star$ is never eliminated.
  At iteration $t$, one of the following two conditions must hold: either there exists $M \in \mathcal{M}_{t}, h_t \leq H$ such that $\mathcal{W}(\pi_{\mathrm{explore}}^t, M, h_t) > \frac{\epsilon}{4H^2|\mathcal{A}|^2}$, or the algorithm returns $\pi_\mathrm{exploit}$ such that $v_{\pi_{\mathrm{exploit}}} > v_{\pi^\star} - \epsilon$.

\end{lemma}

The above two lemmas state that at any time step, the agent either reduces its set of candidate models or finds an exploitation policy which is close to optimal. However, since the initial set of candidate models may be very large, we need to ensure that many models are discarded at each exploration step. Our next lemma bounds the number of iterations of the algorithm by showing that the set of candidate models is reduced by a constant factor at every step.

\begin{lemma}
  \label{iteration-complexity}
  (Iteration Complexity)
  Let $d=\max_{1 \leq h \leq H} rank(A_h)$ and $\phi=\frac{\epsilon}{24H^2|\mathcal{A}|^2\sqrt{d}}$. Suppose that $|\widetilde{\mathcal{W}}(\pi_\mathrm{explore}^t, M, h) - \mathcal{W}(\pi_\mathrm{explore}^t, M, h)| \leq \phi$ holds for all $t$, $h \leq H$ and $M \in \mathcal{M}$. Then the number of rounds of Algorithm \ref{alg:main} with the \texttt{UpdateModelSet} routine given by Algorithm \ref{alg:update-version-space} is at most $Hd \log(\frac{\beta}{2\phi}) / \log(5/3)$.
\end{lemma}

The proof operates by representing each matrix $A_h$ in factored form, which induces an embedding of each model in $\mathcal{M}_t$ in a $d$-dimensional space.
Minimum volume ellipsoids are then constructed around these embeddings.
A geometric argument shows that the volume of these ellipsoids shrinks by a constant factor from one iteration of the algorithm to the next, leading to a number of updates linear in $d$.
Combining the previous lemmas with a concentration argument, we get our main result:
\begin{theorem}
  \label{sample-complexity}
  Assuming that $M^\star \in \mathcal{M}$, for any $\epsilon, \delta \in (0, 1]$ set $\phi=\frac{\epsilon}{24H^2|\mathcal{A}|^2\sqrt{d}}$ and denote $T=Hd \log(\frac{\beta}{2\phi}) / \log(5/3)$. Run Algorithm \ref{alg:main} with inputs $(\mathcal{M}, n, \phi)$ where $n=\Theta(H^4|\mathcal{A}|^4d \log(T|\mathcal{M}||\Pi|/\delta)/\epsilon^2)$, and the \texttt{UpdateModelSet} routine is given by Algorithm \ref{alg:update-version-space}. Then with probability at least $1-\delta$, Algorithm \ref{alg:main} outputs a policy $\pi_{\mathrm{exploit}}$ such that $v_{\pi_{\mathrm{exploit}}} \geq v_{\pi^\star} - \epsilon$.
    The number of trajectories collected is at most $\tilde{O}\Big( \frac{H^5d^2|\mathcal{A}|^4}{\epsilon^2} \log\Big(\frac{T|\mathcal{M}||\Pi|}{\delta} \Big) \Big)$.
\end{theorem}

Note that the above result requires knowledge of $d$ to set the $\phi$ and $n$ parameters. If this quantity is unknown, it can be estimated using a doubling trick which does not affect the algorithm's asymptotic sample complexity. Details can be found in Appendix \ref{doubling-trick}.

\section{Neural-E$^3$: A Practical Instantiation}

The above analysis shows that Algorithm 1 is sample efficient given an idealized instantiation, which may not be computationally practical for large model classes.
Here we give a computationally efficient instantiation called Neural-E$^3$, which requires implementing the \texttt{UpdateModelSet} routine and the planning routines.

\subsection{Model Updates}

We represent $\mathcal{M}_t$ as an ensemble of action-conditional dynamics models $\{M_1, ..., M_E\}$, parameterized by neural networks, which are trained to model the next-state distribution $P_{M^\star}(s_{h+1} | s_h, a)$ using the data from the replay buffer $\mathcal{R}$. The models are trained to minimize the following loss:

\begin{align*}
\mathcal{L}(M, \mathcal{R}) = \mathbb{E}_{(s_{h+1}, a_h, s_h) \sim \mathcal{R}} [-\log P_M(s_{h+1} | s_h, a_h)]
\end{align*}

The models in $\mathcal{M}_1$ are initialized with random weights and the subroutine \texttt{UpdateModelSet} in Algorithm \ref{alg:main} takes as input $\mathcal{M}_{t}$, performs $N_{\mathrm{update}}$ gradient updates to each of the models using different minibatches sampled from $\mathcal{R}$, and returns the updated set of models $\mathcal{M}_{t+1}$.
The dynamics models can be deterministic or stochastic (for example, Mixture Density Networks \cite{Bishop94mixturedensity} or Variational Autoencoders \cite{VAE}).

\subsection{Planning}

The exploration and exploitation phases require computing a policy to optimize $v_{\mathrm{explore}}$ or $v_{\mathrm{exploit}}$ and executing it in the environment. If the environment is deterministic, policies can be represented as action sequences, in which case we use a generalized version of breadth-first search applicable in continuous state spaces. This uses a priority queue, where expanded states are assigned a priority based on their minimum distance to other states in the currently expanded search tree. Details can be found in Appendix \ref{planning-deterministic}.
For stochastic environments, we used implicit policies obtained using Monte-Carlo Tree Search (MCTS) \cite{MCTS}, where each node in the tree consists of empirical distributions predicted by the different models conditioned on the action sequence leading to the node. The agent only executes the first action of the sequence returned by the planning procedure, and replans at every step to account for the stochasticity of the environment. See Appendix \ref{planning-stochastic} for details.

\subsection{Exploitation with Off-Policy RL}

For some problems with sparse rewards, it may be computationally impractical to use planning during the exploitation phase, even with a perfect model.
Note that much of the exploration phase, which uses model disagreement as a fictitious reward, can be seen as an MDP with dense rewards, while the rewards in the true MDP may be sparse.
In these settings, we use an alternative approach where a parameterized value function such as a DQN \cite{mnih2015humanlevel} is trained using the experience collected in the replay buffer during exploration. 
This can be done offline without collecting additional samples from the environment.
We also found this useful for problems with antishaped rewards, where the MCTS procedure can be biased away from the optimal actions if they temporarily lead to lower reward than suboptimal ones.

\subsection{Relationship between Idealized and Practical Algorithms}


For both the idealized and practical algorithms, $\mathcal{M}_t$ represents a set of models with low error on the current replay buffer. In the idealized algorithm, models with high error are eliminated explicitly in Algorithm 2, while in the practical algorithm, models with high error are avoided by the optimization procedure. The main difference between the two algorithms is that the idealized version maintains \textit{all} models in the model class which have low error (which includes the true model), whereas the practical version only maintains a subset due to time and memory constraints. A potential failure mode of the practical algorithm would be if all the models wrongly agree in their predictions in some unexplored part of the state-action space which leads to high reward. However, in practice we found that using different initializations and minibatches was sufficient to obtain a diverse set of models, and that using even a relatively small ensemble ($4$ to $8$ models) led to successful exploration.



\section{Related Work}
Theoretical guarantees for a number of model-based RL algorithms exist in the tabular setting \cite{Kearns2002, Brafman2003, Strehl2005, Sorg2010} and in the continuous setting when the dynamics are assumed to be linear \cite{Sutton2008, abbasi-yadkori11, dean2017}.
The Metric-$E^3$ algorithm \cite{Kakade2003} operates in general state spaces, but its sample complexity depends on the covering number which may be exponential in dimension.
The algorithm of \cite{luo2018} addresses general model classes and optimizes lower bounds on the value function, and provably converges to a \textit{locally} optimal policy with a number of samples polynomial in the dimension of the state space. It also admits an approximate instantiation which was shown to work well in continuous control tasks. The work of \cite{Sun2018} provides an algorithm which provably recovers a \textit{globally} near-optimal policy with polynomial sample complexity using a structural complexity measure which we adapt for our analysis, but does not investigate practical approximations.
The algorithm we analyze is fundamentally different from both of these approaches, as it uses disagreement over predicted states rather than optimism to drive exploration.

Our practical approximation is closely related to the MAX algorithm \cite{Shyam2018}, which also uses disagreement between different models in an ensemble to drive exploration.
Our version differs in a few ways i) we use maximal disagreement rather than variance to measure uncertainty, as this reflects our theoretical analysis ii) we define the exploration MDP differently, by propagating the state predictions of the different models rather than sampling at each step iii) we explicitly address the exploitation step, whereas they focused primarily on exploration.
The work of \cite{pathak19disagreement} also used disagreement between single-step predictions to train an exploration policy.

Several works have empirically demonstrated the sample efficiency of model-based RL in continuous settings \cite{Atkeson97, Pilco, UPN, Nagabandi2018, Chua2018}, including with high-dimensional images \cite{Planet, henaff2018, WorldModels}.
These have primarily focused on settings with dense rewards where simple exploration was sufficient, or where rich observational data was available.

Other approaches to exploration include augmenting rewards with exploration bonuses, such as inverse counts in the tabular setting \cite{Strehl2008, Kolter2009}, pseudo-counts derived from density models over the state space \cite{Bellemare16, Ostrovski2017}, prediction errors of either a dynamics model \cite{Pathak2017} or a randomly initialized network \cite{RND}, or randomizing value functions \cite{BootDQN, Osband2017}. These have primarily focused on model-free methods, which have been known to have high sample complexity despite yielding good final performance.




\section{Experiments}

We now give empirical results for the Neural-E$^3$ algorithm described in Section 4. See Appendix \ref{appendix-experiments} for experimental details and \url{https://github.com/mbhenaff/neural-e3} for source code.

\begin{figure}[t]
  \centering
  \subfigure[]{\fbox{\includegraphics[width=0.42\textwidth, height=0.1\textwidth]{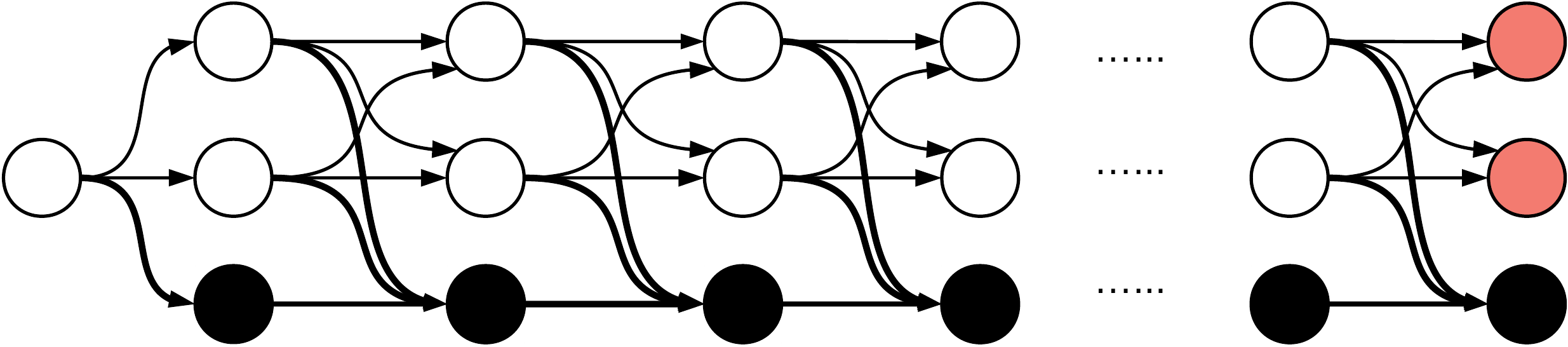}}\label{fig:lock}}
\subfigure[]{\fbox{\includegraphics[width=0.11\textwidth, height=0.1\textwidth]{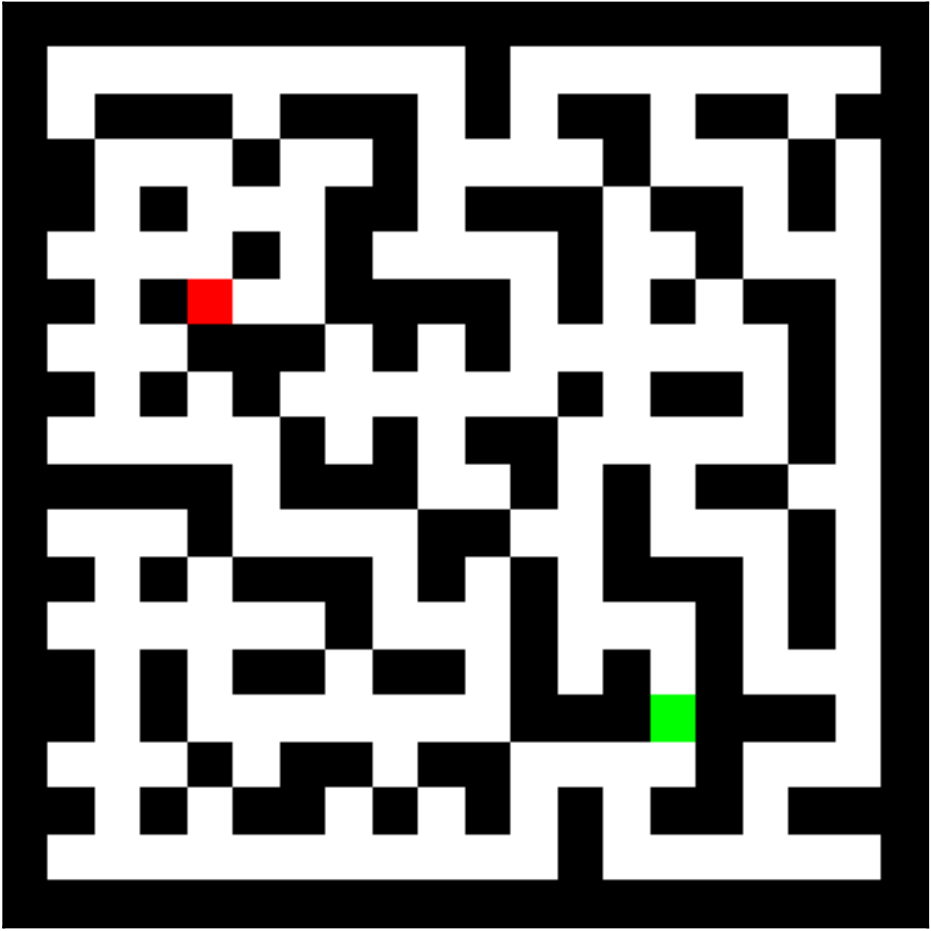}
\includegraphics[width=0.11\textwidth, height=0.1\textwidth]{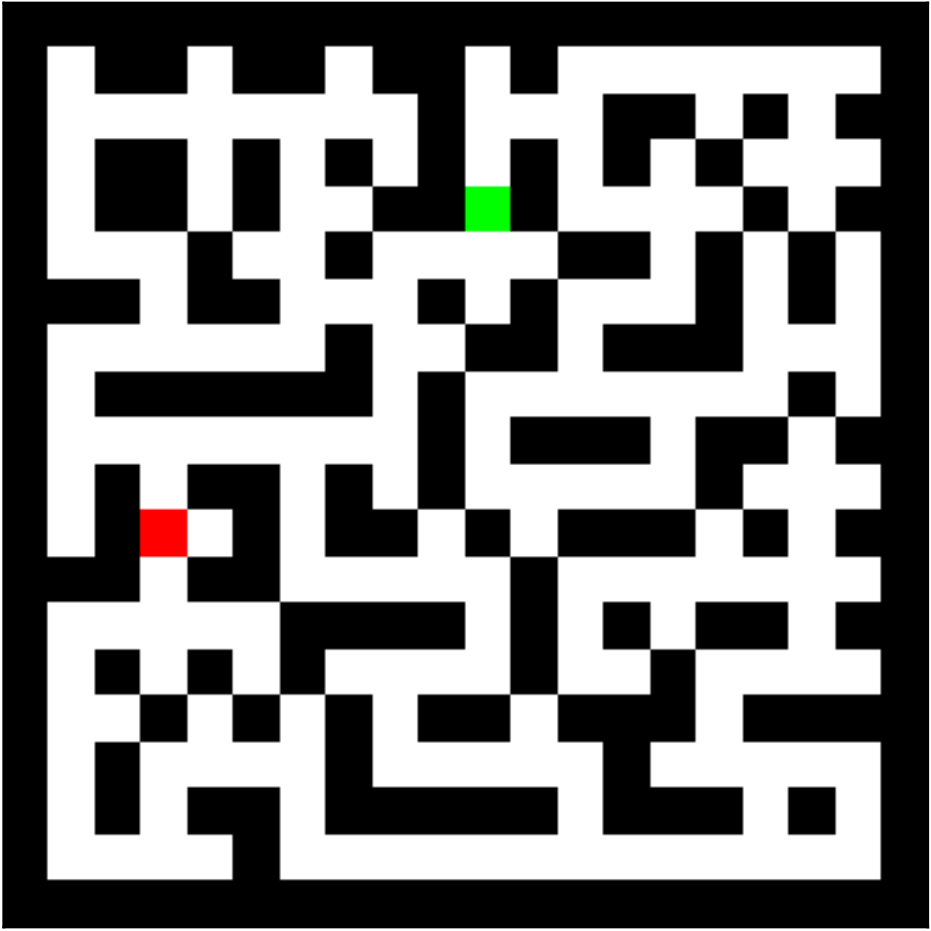}}\label{fig:maze}}
\subfigure[]{\fbox{\includegraphics[width=0.12\textwidth, height=0.1\textwidth]{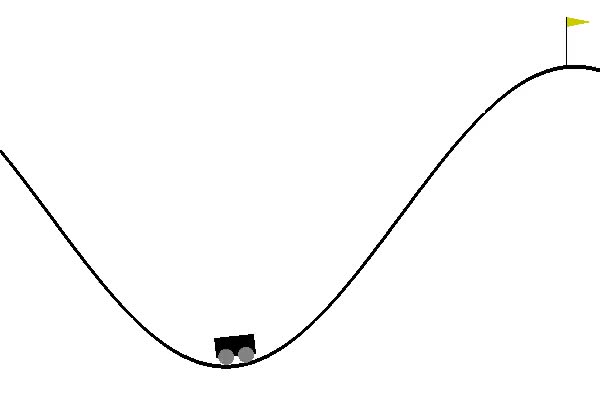}}
\fbox{\includegraphics[width=0.12\textwidth, height=0.1\textwidth]{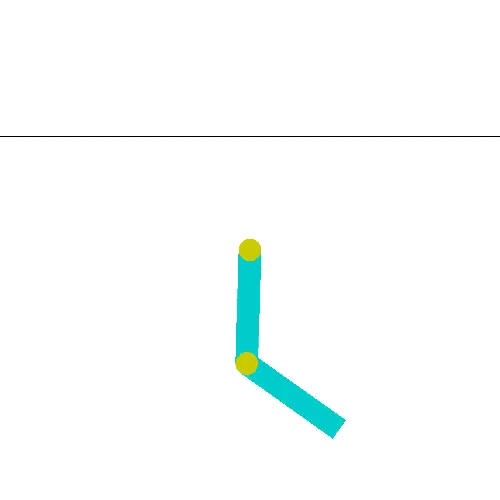}}\label{fig:control}}
\caption{Environments tested. a) \textbf{Stochastic combination lock:} The agent must reach the red states to collect high reward while avoiding the dead states (black) from which it cannot recover. b) \textbf{Mazes:} The agent (green) must navigate through the maze to reach the goal (red). Different mazes are generated each episode, which requires generalizing across mazes (colors are changed here for readability) c) \textbf{Continuous Control}: Classic control tasks requiring non-trivial exploration.}
\label{fig:envs}
\end{figure}

\subsection{Stochastic Combination Lock}

We begin with a set of experiments on the stochastic combination lock environment described in \cite{PCID} and shown in Figure \ref{fig:lock}. These environments consist of $H$ levels with 3 states per level and 4 actions. Two of the states lead to high reward and the third is a dead state from which it is impossible to recover. The effect of actions are flipped with probability 0.1, and the one-hot state encodings are appended with random Bernoulli noise to increase the number of possible observations. We experimented with two task variants: a first where the rewards are zero everywhere except for the red states, and a second where small, antishaped rewards encourage the agent to transition to the dead states (see Appendix \ref{appendix:combolock} for details). This tests an algorithm's robustness to poor local optima.

We compare against three other methods: a double DQN \cite{DDQN} with prioritized experience replay \cite{PER} using the OpenAI Baselines implementation \cite{OpenAIBaselines}, a Proximal Policy Optimization (PPO) agent \cite{PPO}, and a PPO agent with a Random Network Distillation (RND) exploration bonus \cite{RND}.
For Neural-E$^3$, we used stochastic dynamics models outputting the parameters of a multivariate Bernoulli distribution, and the MCTS procedure described in Appendix \ref{planning-stochastic} during the exploration phase. We used the DQN-based method described in Section 4.3 for the exploit phase.

Figure \ref{fig:combolock-results} shows performance across 5 random seeds for the first variant of the task. For all horizons, Neural-E$^3$ achieves the optimal reward across most seeds. The DQN also performs well, although it often requires more samples than Neural-E$^3$. For longer horizons, PPO never collects rewards, while PPO+RND eventually succeeds given a large number of episodes (see Appendix \ref{appendix:combolock}).

Figure \ref{fig:combolock-results-antishaped} shows results for the task variant with antishaped rewards. For longer horizons, Neural-E$^3$ is the only method to achieve the globally optimal reward, whereas none of the other methods get past the poor local optimum induced by the misleading rewards. Note that Neural-E$^3$ actually obtains \textit{less} reward than the other methods during its exploration phase, but this pays off during exploitation since it enables the agent to eventually discover states with much higher reward.

\subsection{Maze Environment}

We next evaluated our approach on a maze environment, which is a modified version of the Collect domain \cite{VPN}, shown in Figure \ref{fig:maze}. States consist of RGB images where the three channels represent the walls, the agent and the goal respectively. The agent receives a reward of $2.0$ for reaching the goal, $-0.5$ for hitting a wall and $-0.2$ otherwise.
Mazes are generated randomly for each episode, thus the number of states is extremely large and the agent must learn to generalize across mazes.
Our dynamics models are action-conditional convolutional networks taking as input an image and action and predicting the next image and reward. We used the deterministic search procedure described in Section \ref{planning-deterministic} for planning.

We compared to two other approaches. The first was a double DQN with prioritized experience replay as before. The second was a model-based agent identical to ours, except that it uses a uniform exploration policy during the explore phase. This is similar to the PETS algorithm \cite{Chua2018} applied to discrete action spaces, as it optimizes rewards over an ensemble of dynamics models. We call this UE$^2$, for Uniform Explore Exploit.

Performance measured by reward across 3 random seeds is shown in Figure \ref{fig:maze-results} for different maze sizes.
The DQN agent is able to solve the smallest $5 \times 5$ mazes after a large number of episodes, but is not able to learn meaningful behavior for larger mazes.
The UE$^2$ and Neural-E$^3$ agents both perform similarly for the $5 \times 5$ mazes, but the relative performance of Neural-E$^3$ improves as the size of the maze becomes larger.
Note also that the Neural-E$^3$ agent collects more reward during its exploration phase, even though it is not explicitly optimizing for reward but rather for model disagreement.
Figure \ref{fig:predictions-crop} in Appendix \ref{appendix-maze} shows the model predictions for an action sequence executed by the Neural-E$^3$ agent during the exploration phase. The predictions of the different models agree until the reward is reached, which is a rare event.



\begin{figure}
  \centering


  \subfigure[Stochastic Combination Lock]{
    \includegraphics[width=0.19\textwidth]{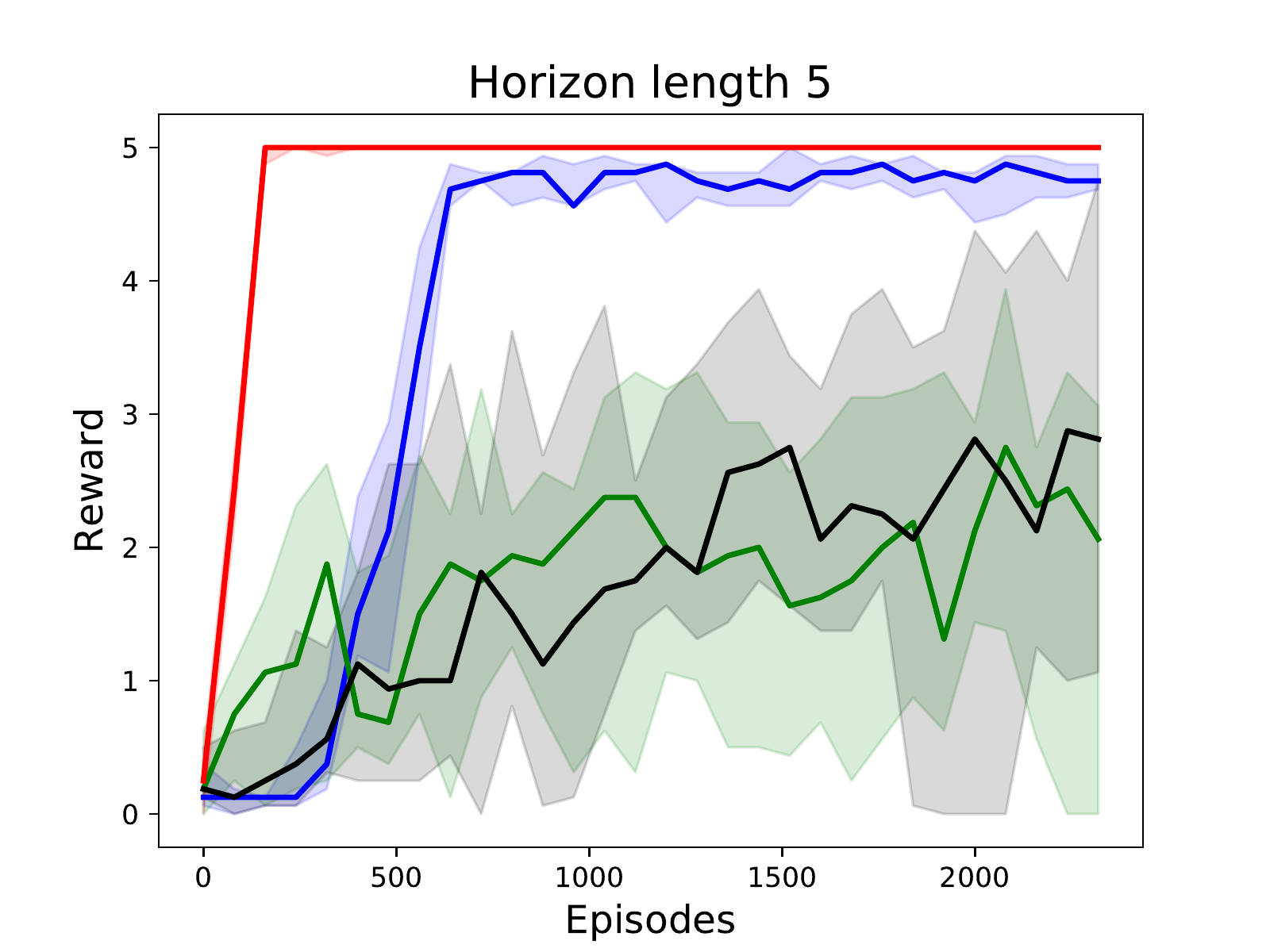}
    \includegraphics[width=0.19\textwidth]{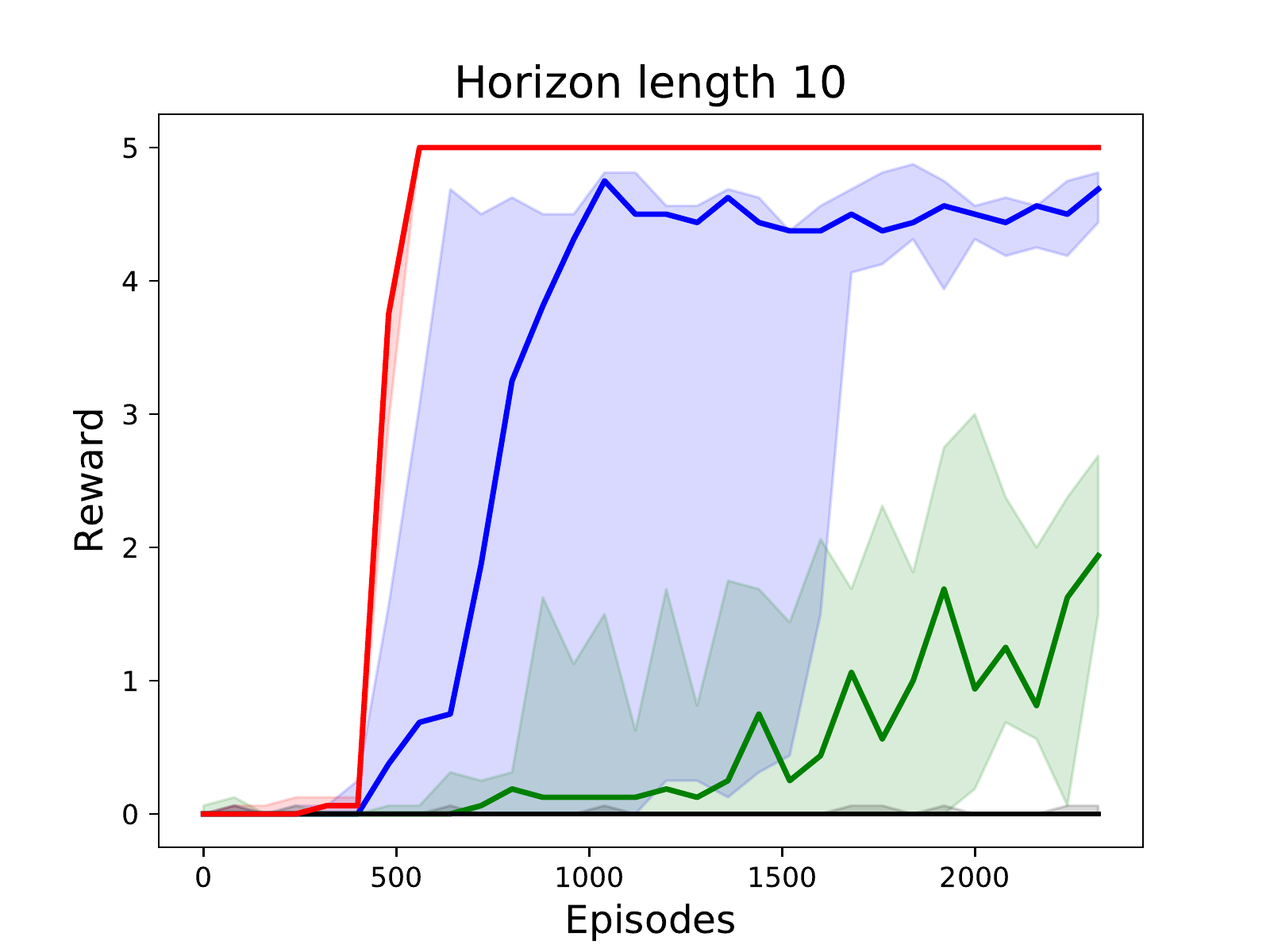}
    \includegraphics[width=0.19\textwidth]{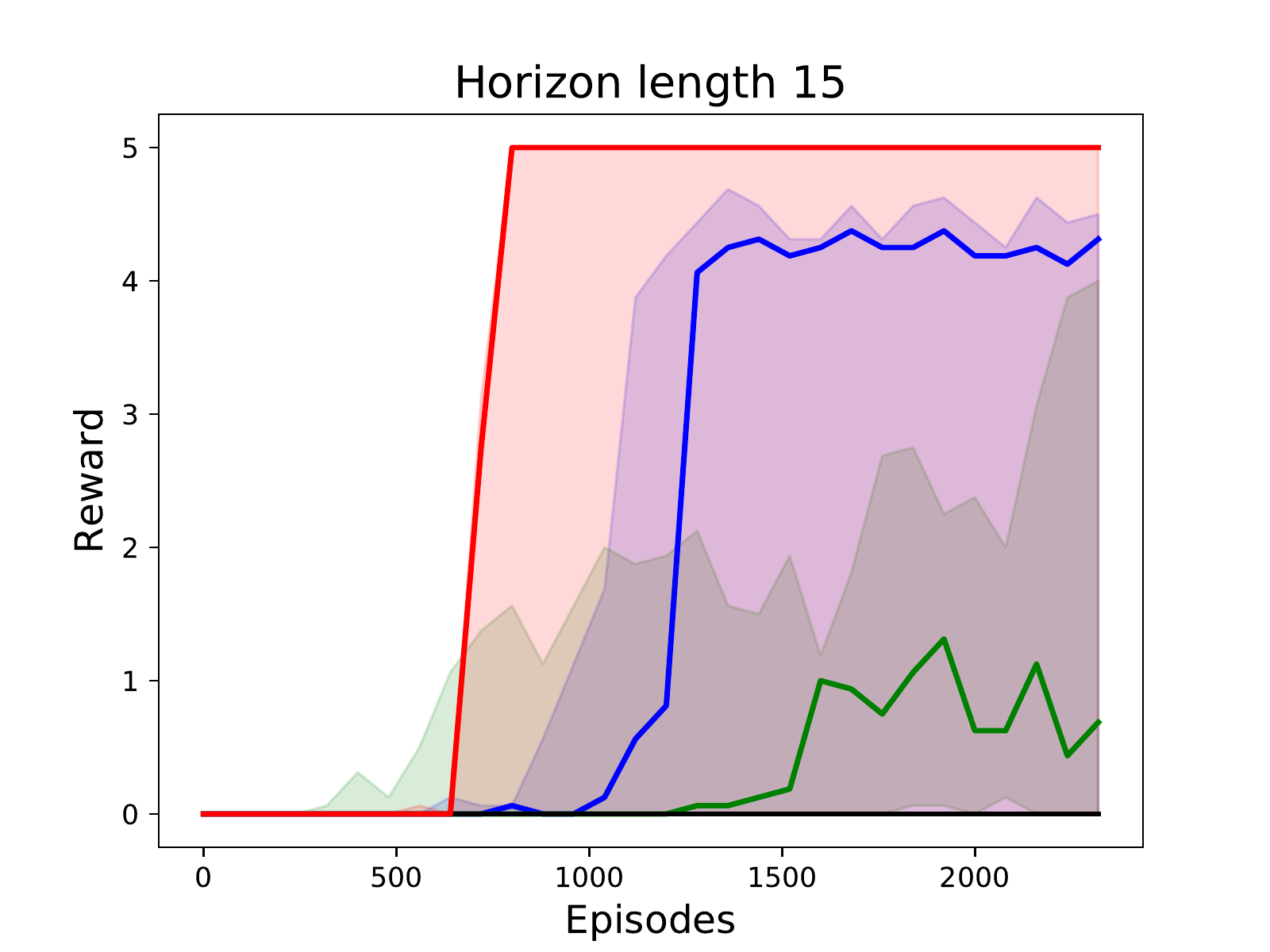}
    \includegraphics[width=0.19\textwidth]{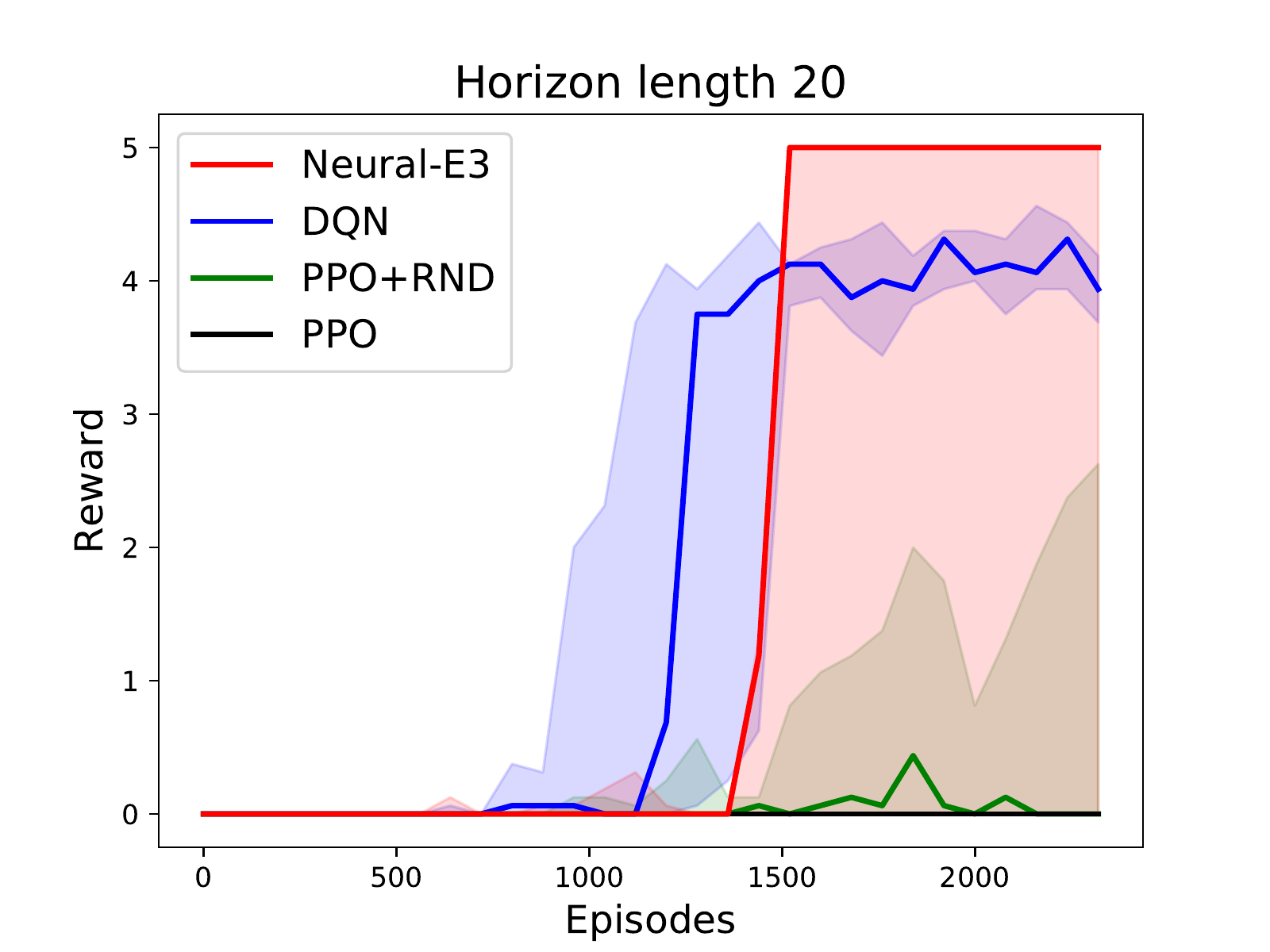}\label{fig:combolock-results}} \\

  \subfigure[Stochastic Combination Lock with antishaped rewards]{
    \includegraphics[width=0.19\textwidth]{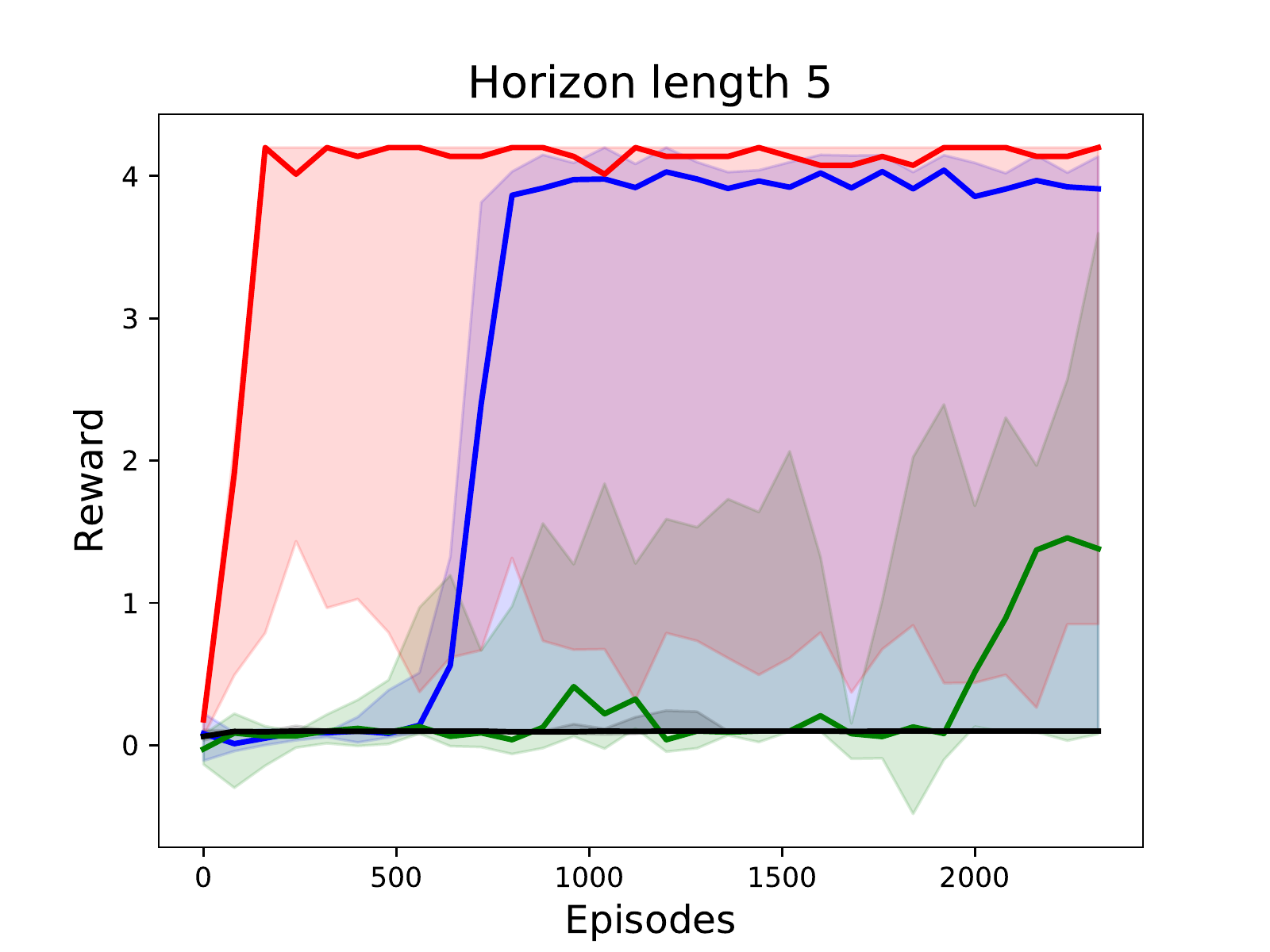}
    \includegraphics[width=0.19\textwidth]{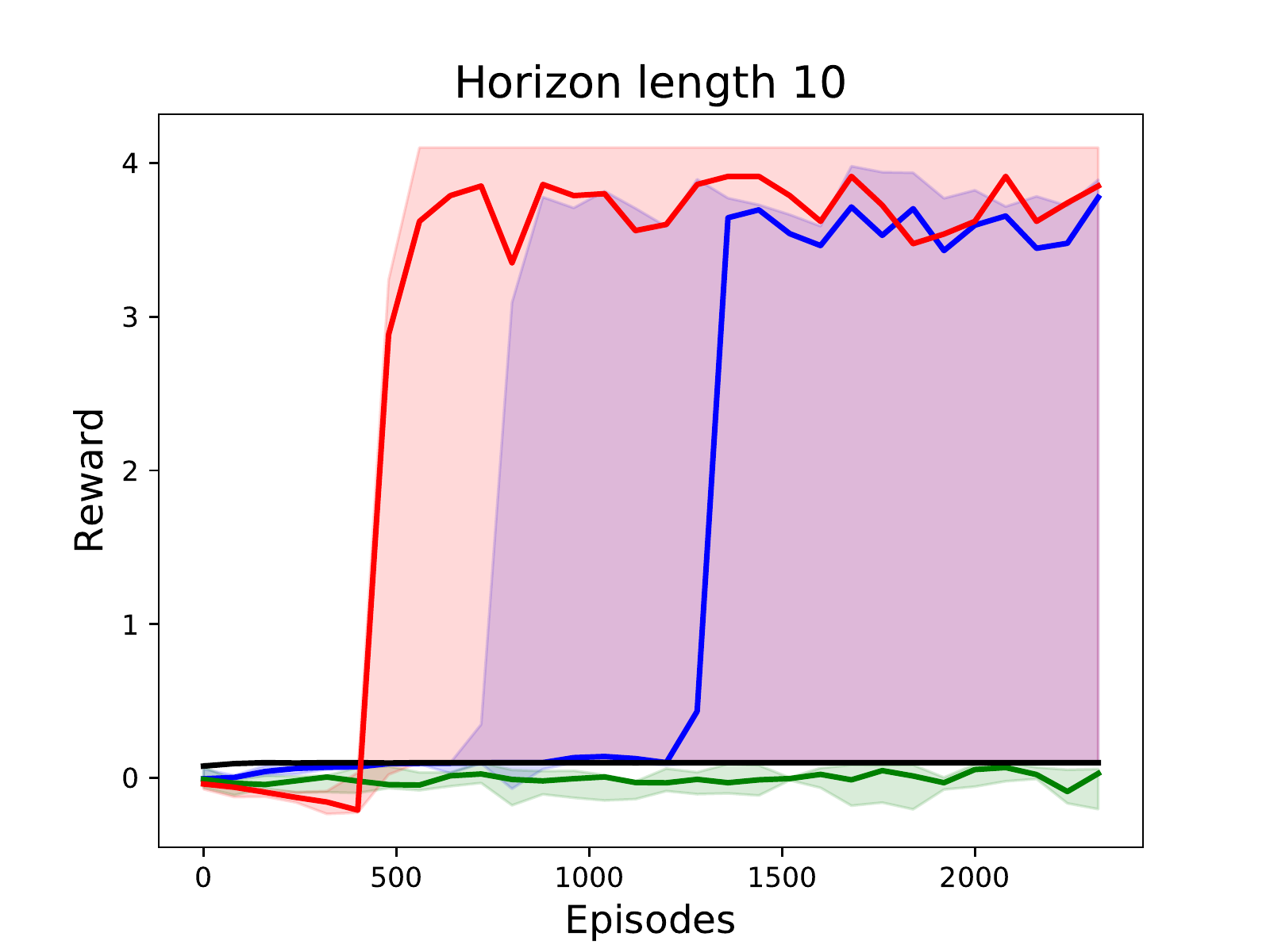}
    \includegraphics[width=0.19\textwidth]{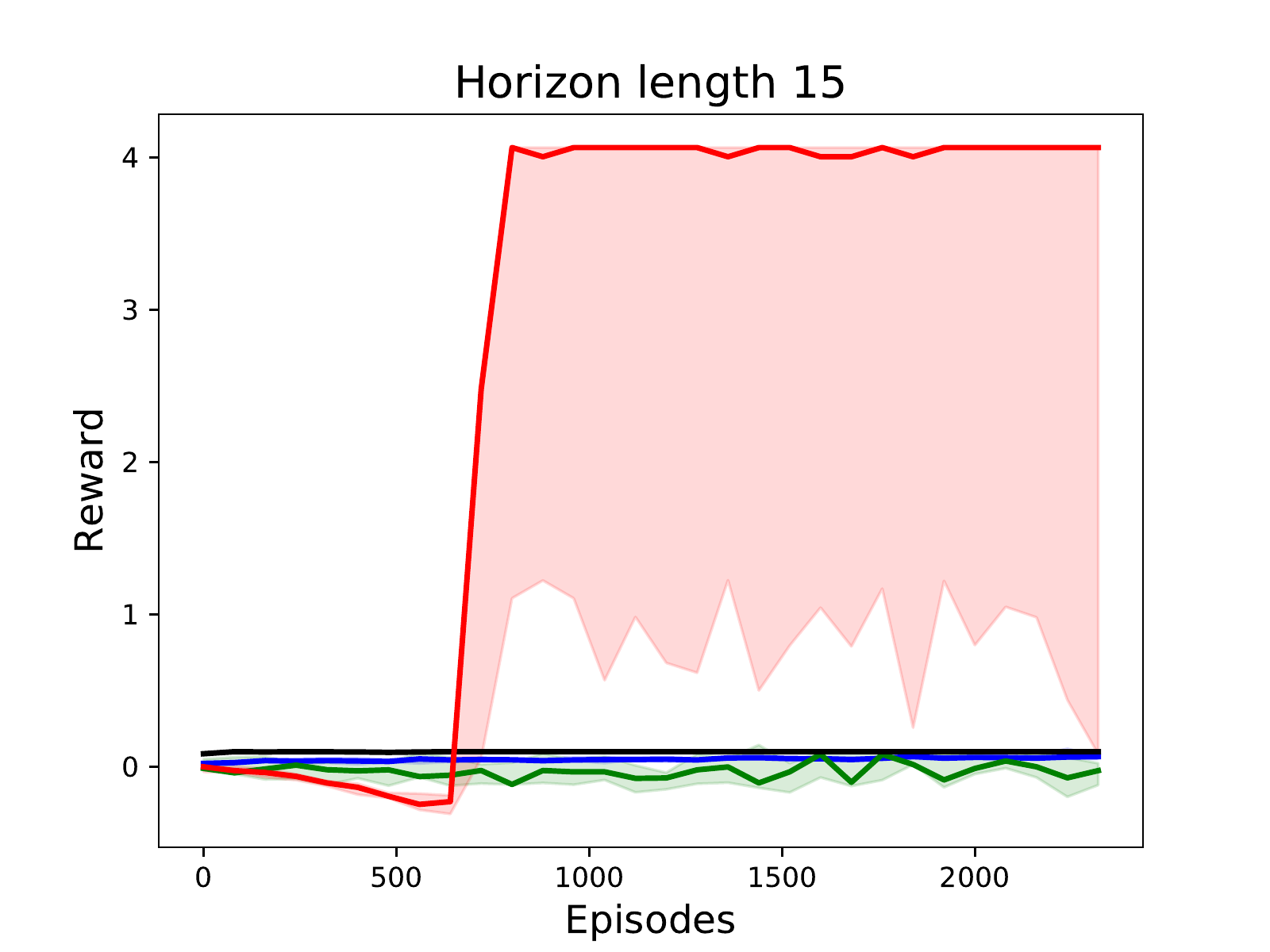}
    \includegraphics[width=0.19\textwidth]{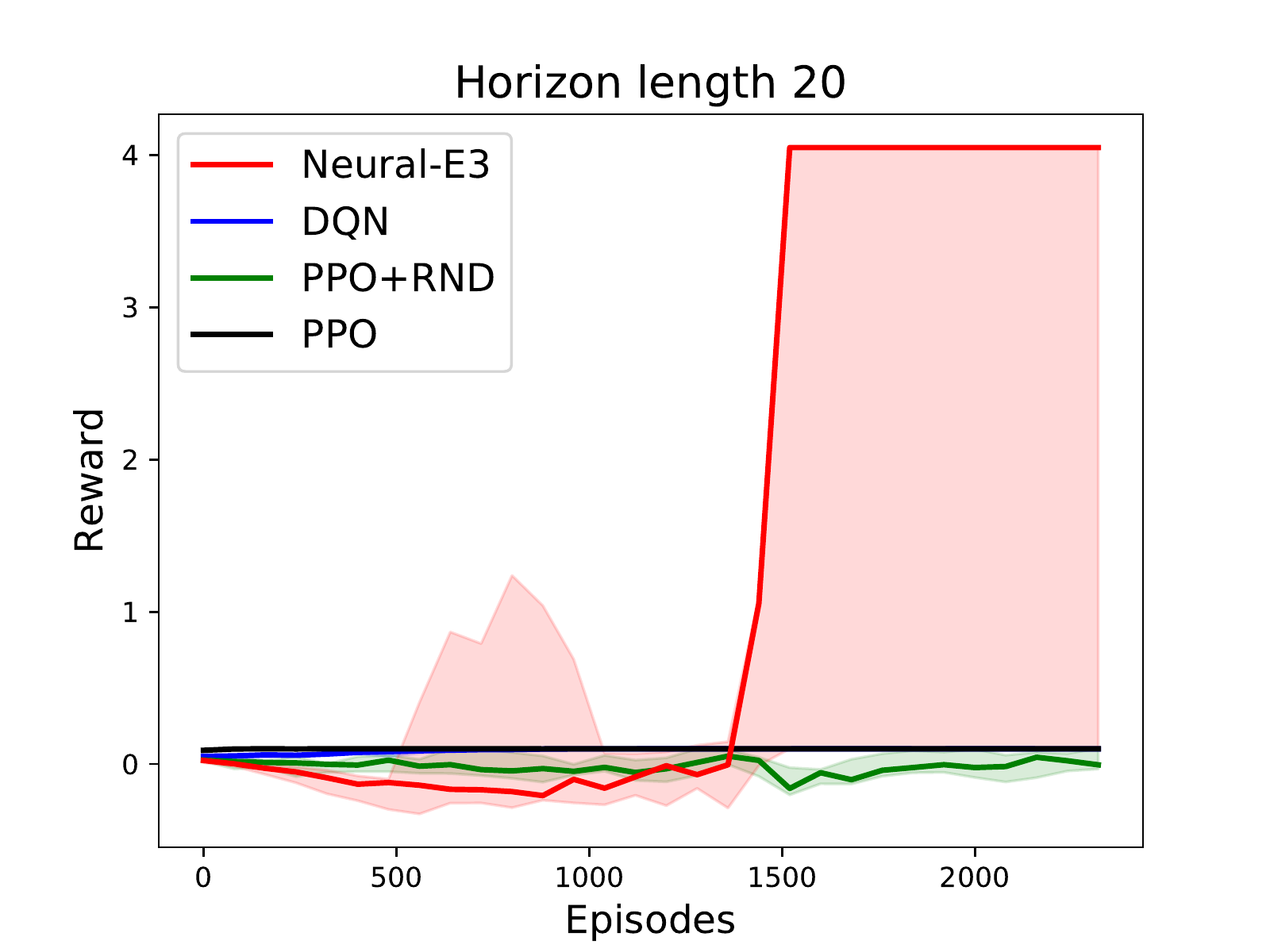}\label{fig:combolock-results-antishaped}} \\


\subfigure[Maze]{\includegraphics[width=0.19\textwidth]{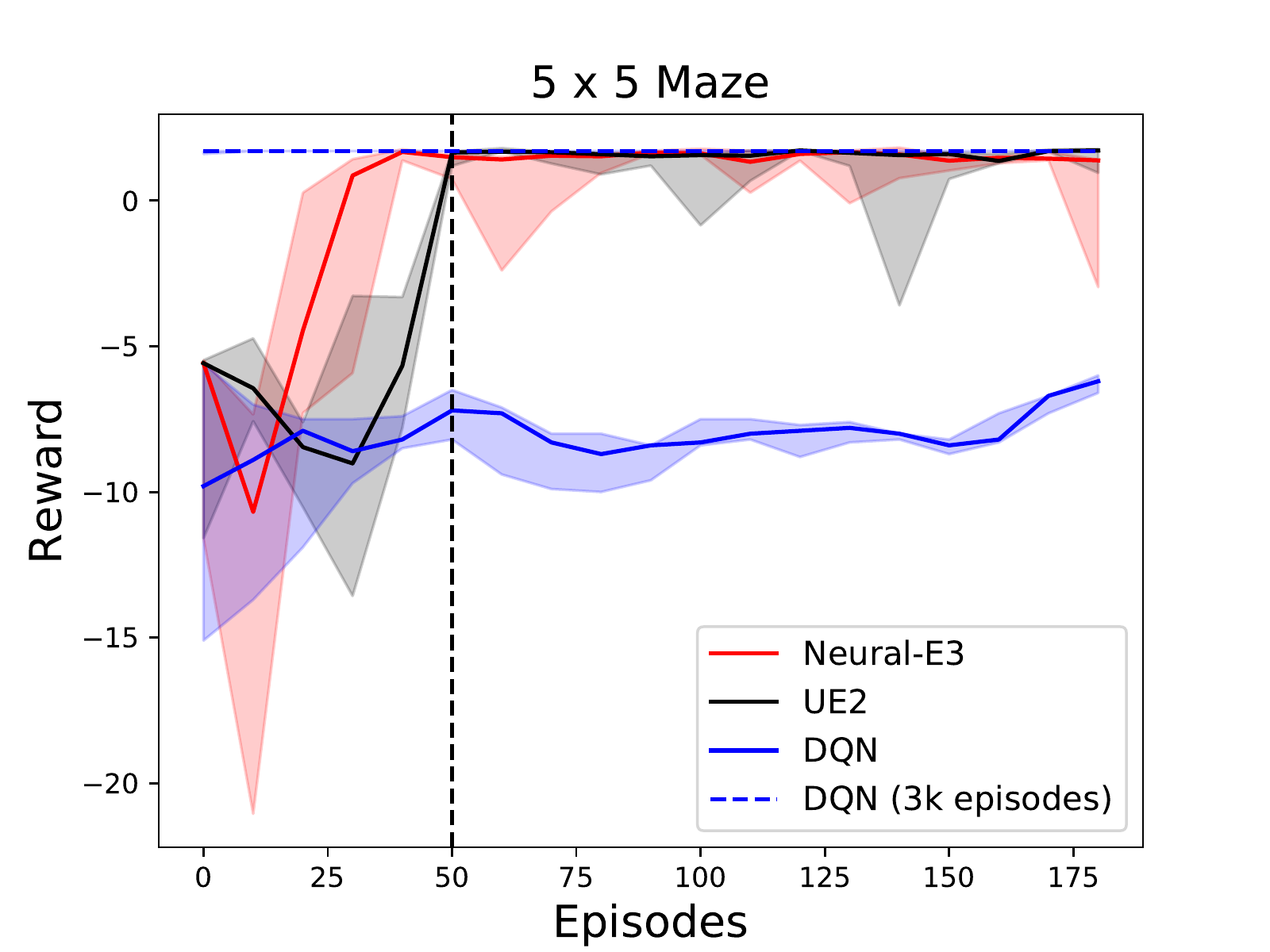}
\includegraphics[width=0.19\textwidth]{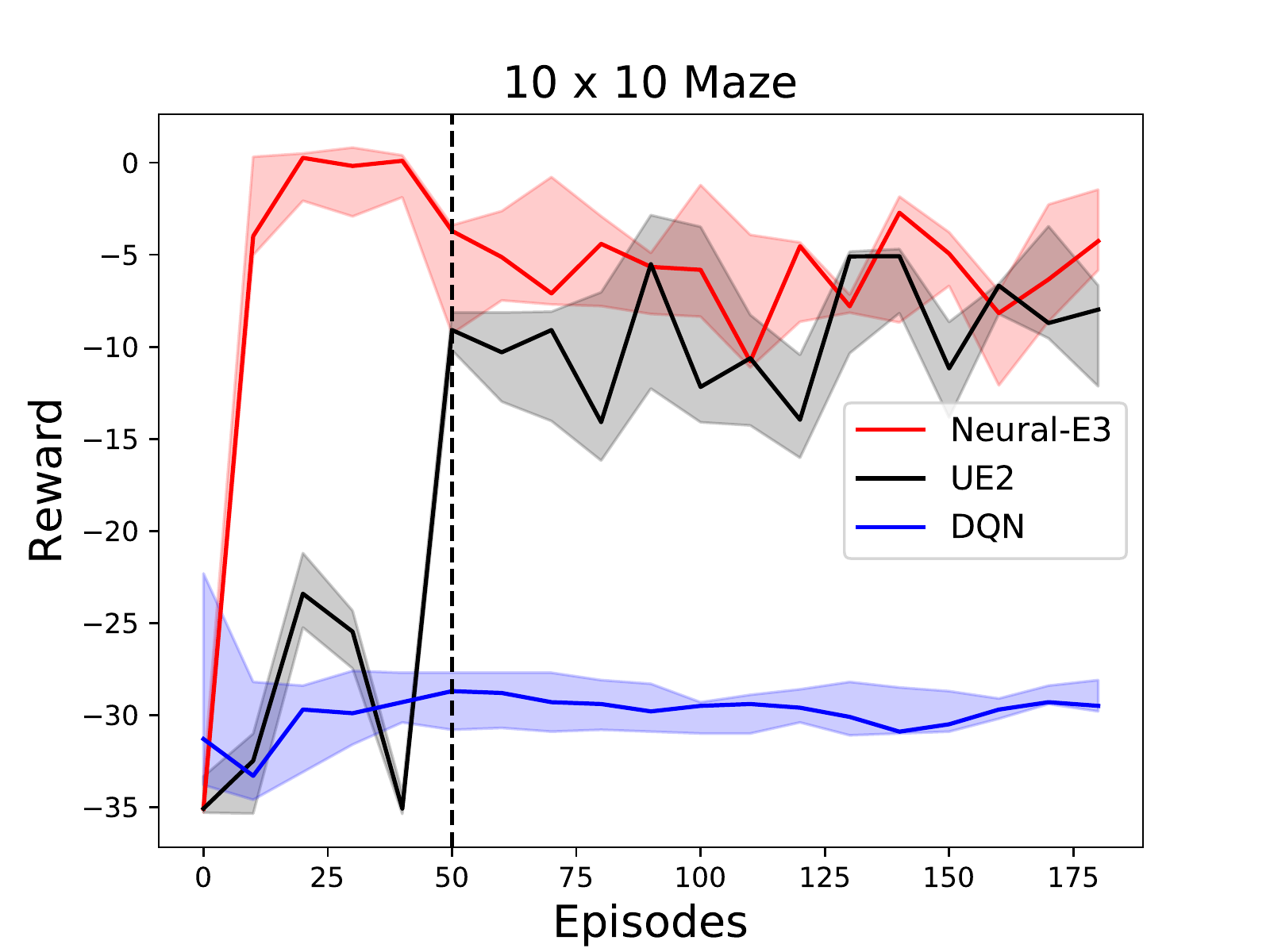}
\includegraphics[width=0.19\textwidth]{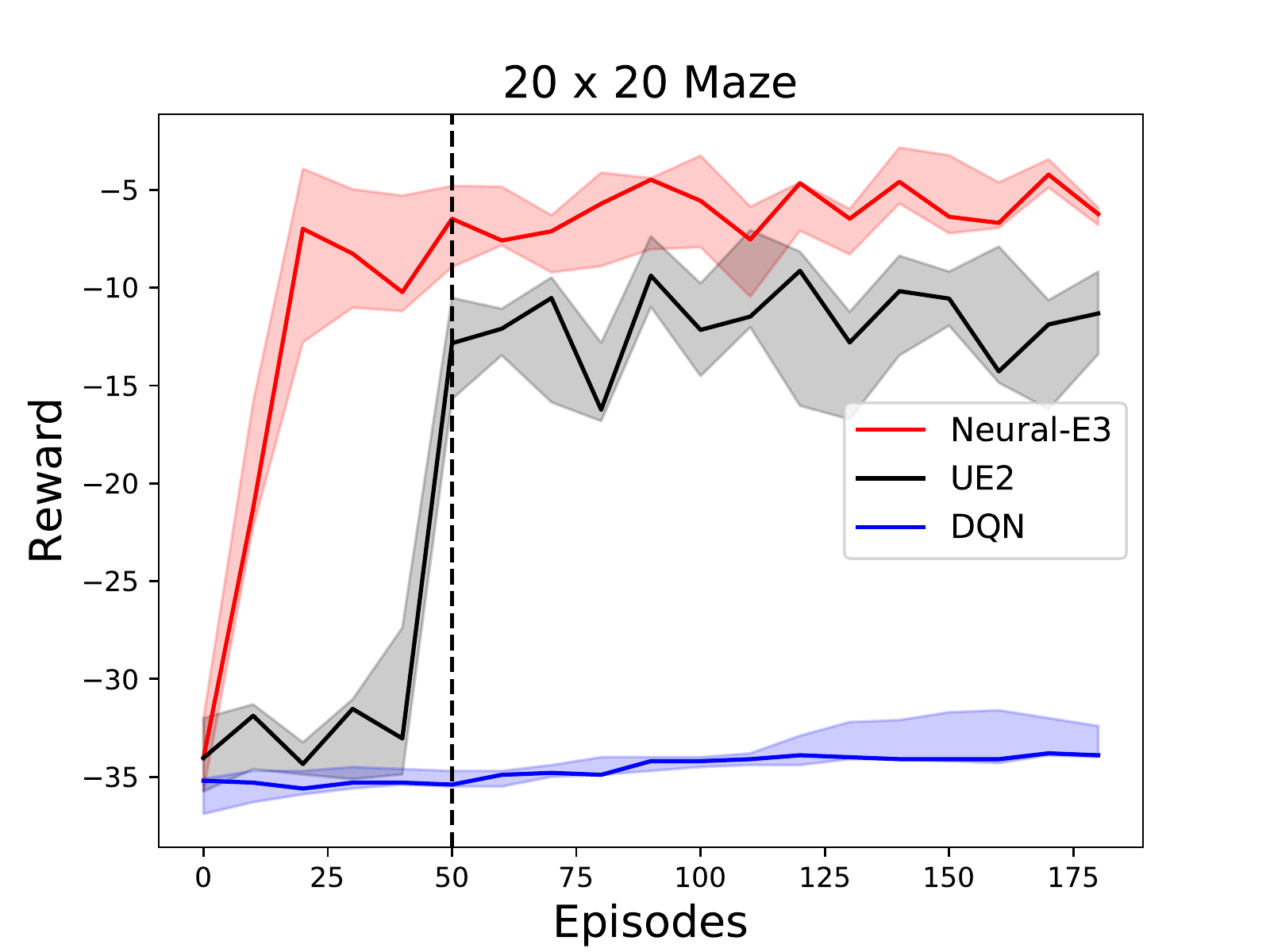}\label{fig:maze-results}}
\rulesep
\subfigure[Classic Control]{\includegraphics[width=0.19\textwidth]{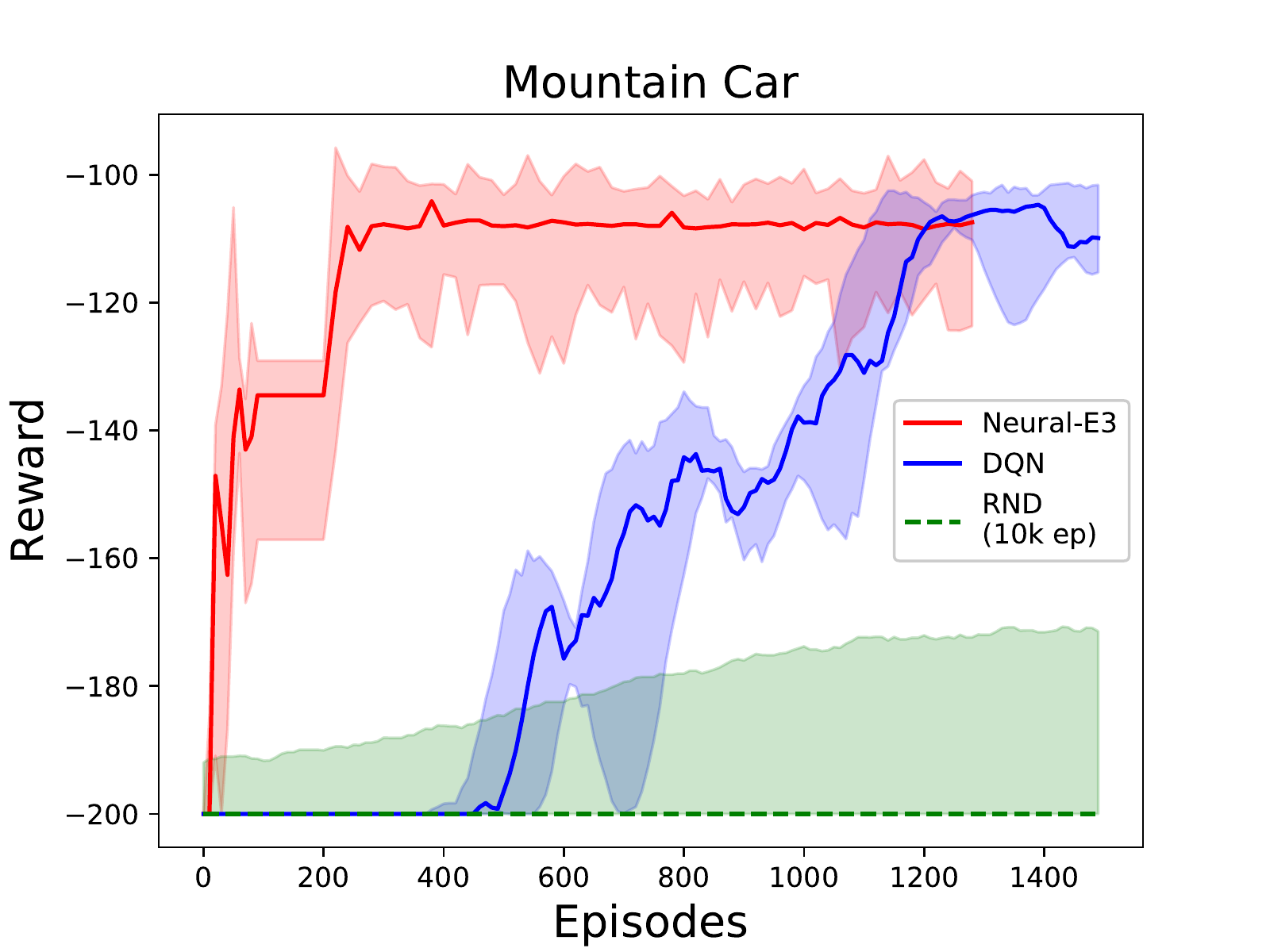}
  \includegraphics[width=0.19\textwidth]{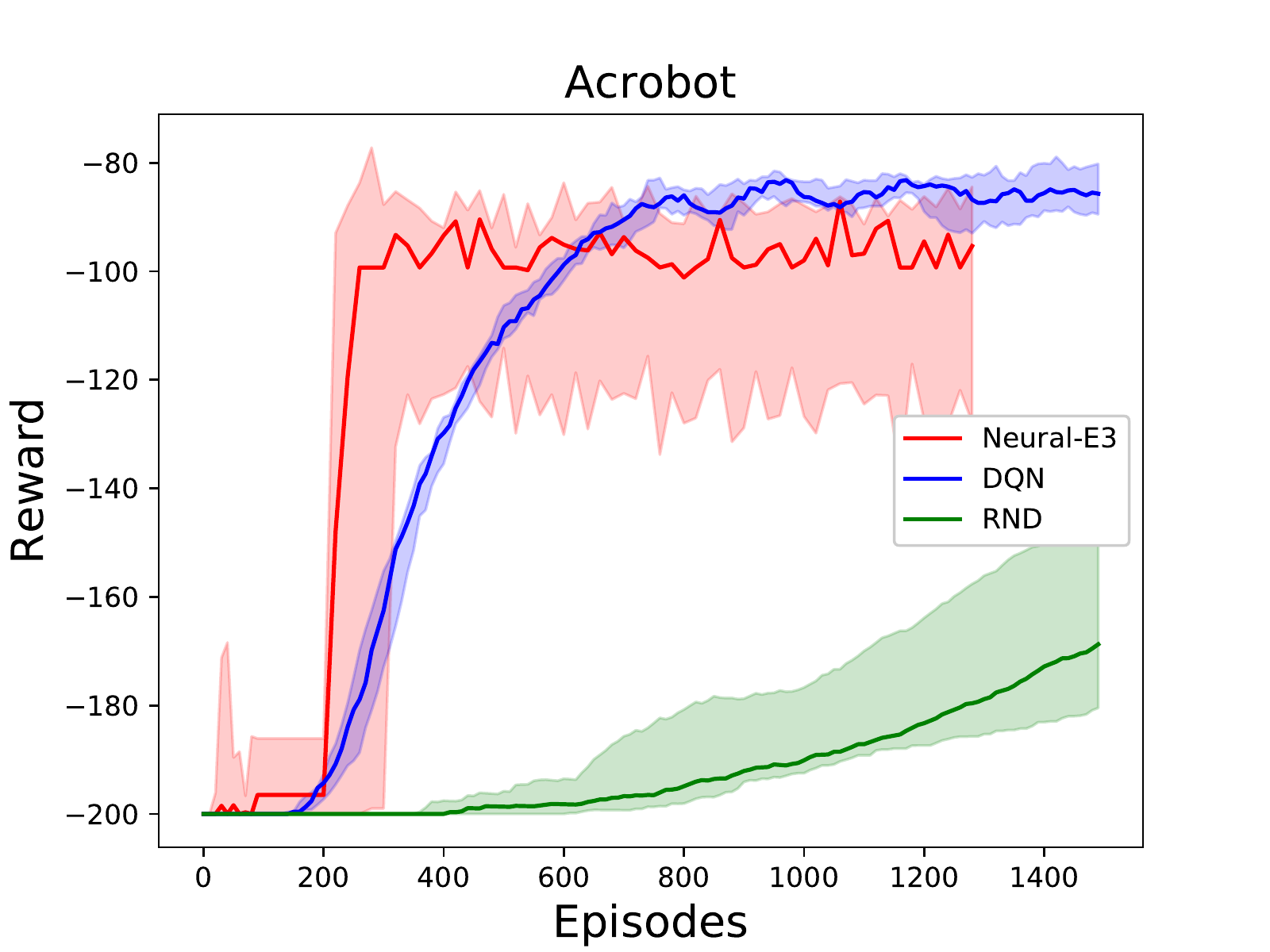}\label{fig:control-results}}

\caption{Comparison of methods across different domains. Solid lines represent median performance across seeds, shaded region represents range between best and worst seeds.}
\label{fig:all-results}
\end{figure}

\subsection{Continuous Control}

We then evaluated our approach on two continuous control domains, shown in Figure \ref{fig:control}. MountainCar \cite{Moore90} is an environment with simple non-linear dynamics and continuous state space $(\mathcal{S} \subseteq \mathbb{R}^2)$ where the agent must drive an underpowered car up a steep hill, which requires building momentum by first driving up the opposite end of the hill.
The agent only receives reward at the top of the hill, hence this requires non-trivial exploration.
Acrobot \cite{Acrobot} requires swinging a simple under-actuated robot above a given height, also with a continuous state space $(\mathcal{S} \subseteq \mathbb{R}^6)$. Both tasks have discrete action spaces with $|\mathcal{A}|=3$.

We found that even planning with a perfect model was computationally impractical due to the sparsity of rewards, hence we used the method described in section 4.3, where we trained a DQN offline using the data collected during exploration.
Results for Neural-E$^3$, DQN and RND agents across 5 random seeds are shown in Figure \ref{fig:control-results}. For Mountain Car, the DQN is able to solve the task but requires around 1200 episodes to do so. Neural-E$^3$ is able to quickly explore, and solves the task to a similar degree of success in under 300 episodes. The RND agent only starts to collect reward after 10K episodes. For the Acrobot task, Neural-E$^3$ also explores quickly, although its increase in sample efficiency is less pronounced compared to the DQN.
The RND agent is also able to make quicker progress on this task, which suggests that the exploration problem may not be as difficult.

\section{Conclusion}

This work extends the classic $E^3$ algorithm to operate in large or infinite state spaces.
On the theoretical side, we present a model-elimination based version of the algorithm which provably requires only a polynomial number of samples to learn a near-optimal policy with high probability.
Empirically, we show that this algorithm can be approximated using neural networks and still provide good sample efficiency in practice.
An interesting direction for future work would be combining the exploration and exploitation phases in a unified process, which has been done in the tabular setting \cite{Brafman2003}.
Another direction would be to explicitly encourage disagreement between different models in the ensemble for unseen inputs, in order to better approximate the maximal disagreement between models in a version space which we use in our idealized algorithm. Such ideas have been proposed in active learning \cite{Cohn1994} and contextual bandits \cite{Foster2018}, and could potentially be adapted to multi-step RL.

\clearpage

\subsubsection*{Acknowledgments}

I would like to thank Akshay Krishnamurthy, John Langford, Alekh Agarwal and Miro Dudik for helpful discussions and feedback.

\medskip

\bibliographystyle{abbrv}
\bibliography{neurips_2019}

\begin{thebibliography}{10}

\bibitem{abbasi-yadkori11}
Y.~Abbasi-Yadkori and C.~Szepesvári.
\newblock Regret bounds for the adaptive control of linear quadratic systems.
\newblock In S.~M. Kakade and U.~von Luxburg, editors, {\em Proceedings of the
  24th Annual Conference on Learning Theory}, volume~19 of {\em Proceedings of
  Machine Learning Research}, pages 1--26, Budapest, Hungary, 09--11 Jun 2011.
  PMLR.

\bibitem{Atkeson97}
C.~G. Atkeson and J.~C. Santamaria.
\newblock A comparison of direct and model-based reinforcement learning.
\newblock In {\em IN INTERNATIONAL CONFERENCE ON ROBOTICS AND AUTOMATION},
  pages 3557--3564. IEEE Press, 1997.

\bibitem{Bellemare16}
M.~G. Bellemare, S.~Srinivasan, G.~Ostrovski, T.~Schaul, D.~Saxton, and
  R.~Munos.
\newblock Unifying count-based exploration and intrinsic motivation.
\newblock {\em CoRR}, abs/1606.01868, 2016.

\bibitem{Bishop94mixturedensity}
C.~M. Bishop.
\newblock Mixture density networks.
\newblock Technical report, 1994.

\bibitem{Brafman2003}
R.~I. Brafman and M.~Tennenholtz.
\newblock R-max - a general polynomial time algorithm for near-optimal
  reinforcement learning.
\newblock {\em J. Mach. Learn. Res.}, 3:213--231, Mar. 2003.

\bibitem{Gym}
G.~Brockman, V.~Cheung, L.~Pettersson, J.~Schneider, J.~Schulman, J.~Tang, and
  W.~Zaremba.
\newblock Openai gym, 2016.

\bibitem{RND}
Y.~Burda, H.~Edwards, A.~Storkey, and O.~Klimov.
\newblock Exploration by random network distillation.
\newblock In {\em International Conference on Learning Representations}, 2019.

\bibitem{Chua2018}
K.~Chua, R.~Calandra, R.~McAllister, and S.~Levine.
\newblock Deep reinforcement learning in a handful of trials using
  probabilistic dynamics models.
\newblock In {\em NeurIPS}, 2018.

\bibitem{Cohn1994}
D.~Cohn, L.~Atlas, and R.~Ladner.
\newblock Improving generalization with active learning.
\newblock {\em Machine Learning}, 15(2):201--221, May 1994.

\bibitem{MCTS}
R.~Coulom.
\newblock Efficient selectivity and backup operators in monte-carlo tree
  search.
\newblock volume 4630, 05 2006.

\bibitem{dean2017}
S.~Dean, H.~Mania, N.~Matni, B.~Recht, and S.~Tu.
\newblock On the sample complexity of the linear quadratic regulator.
\newblock {\em CoRR}, abs/1710.01688, 2017.

\bibitem{Pilco}
M.~P. Deisenroth and C.~E. Rasmussen.
\newblock Pilco: A model-based and data-efficient approach to policy search.
\newblock In {\em In Proceedings of the International Conference on Machine
  Learning}, 2011.

\bibitem{OpenAIBaselines}
P.~Dhariwal, C.~Hesse, O.~Klimov, A.~Nichol, M.~Plappert, A.~Radford,
  J.~Schulman, S.~Sidor, Y.~Wu, and P.~Zhokhov.
\newblock Openai baselines.
\newblock \url{https://github.com/openai/baselines}, 2017.

\bibitem{PCID}
S.~Du, A.~Krishnamurthy, N.~Jiang, A.~Agarwal, M.~Dudik, and J.~Langford.
\newblock Provably efficient {RL} with rich observations via latent state
  decoding.
\newblock In {\em Proceedings of the 36th International Conference on Machine
  Learning}, 2019.

\bibitem{Foster2018}
D.~J. Foster, A.~Agarwal, M.~Dud{\'{\i}}k, H.~Luo, and R.~E. Schapire.
\newblock Practical contextual bandits with regression oracles.
\newblock In {\em {ICML}}, volume~80 of {\em Proceedings of Machine Learning
  Research}, pages 1534--1543. {PMLR}, 2018.

\bibitem{WorldModels}
D.~Ha and J.~Schmidhuber.
\newblock Recurrent world models facilitate policy evolution.
\newblock In S.~Bengio, H.~Wallach, H.~Larochelle, K.~Grauman, N.~Cesa-Bianchi,
  and R.~Garnett, editors, {\em Advances in Neural Information Processing
  Systems 31}, pages 2450--2462. Curran Associates, Inc., 2018.

\bibitem{Planet}
D.~Hafner, T.~P. Lillicrap, I.~Fischer, R.~Villegas, D.~Ha, H.~Lee, and
  J.~Davidson.
\newblock Learning latent dynamics for planning from pixels.
\newblock {\em CoRR}, abs/1811.04551, 2018.

\bibitem{DDQN}
H.~v. Hasselt, A.~Guez, and D.~Silver.
\newblock Deep reinforcement learning with double q-learning.
\newblock In {\em Proceedings of the Thirtieth AAAI Conference on Artificial
  Intelligence}, AAAI'16, pages 2094--2100. AAAI Press, 2016.

\bibitem{henaff2018}
M.~Henaff, A.~Canziani, and Y.~LeCun.
\newblock Model-predictive policy learning with uncertainty regularization for
  driving in dense traffic.
\newblock In {\em International Conference on Learning Representations}, 2019.

\bibitem{Hessel2018RainbowCI}
M.~Hessel, J.~Modayil, H.~P. van Hasselt, T.~Schaul, G.~Ostrovski, W.~Dabney,
  D.~Horgan, B.~Piot, M.~G. Azar, and D.~Silver.
\newblock Rainbow: Combining improvements in deep reinforcement learning.
\newblock In {\em AAAI}, 2018.

\bibitem{jiang2017}
N.~Jiang, A.~Krishnamurthy, A.~Agarwal, J.~Langford, and R.~E. Schapire.
\newblock Contextual decision processes with low {B}ellman rank are
  {PAC}-learnable.
\newblock In D.~Precup and Y.~W. Teh, editors, {\em Proceedings of the 34th
  International Conference on Machine Learning}, volume~70 of {\em Proceedings
  of Machine Learning Research}, pages 1704--1713, International Convention
  Centre, Sydney, Australia, 06--11 Aug 2017. PMLR.

\bibitem{Kakade2003}
S.~M. Kakade, M.~Kearns, and J.~Langford.
\newblock Exploration in metric state spaces.
\newblock In {\em ICML}, 2003.

\bibitem{Kearns1999}
M.~Kearns and D.~Koller.
\newblock Efficient reinforcement learning in factored mdps.
\newblock In {\em Proceedings of the 16th International Joint Conference on
  Artificial Intelligence - Volume 2}, IJCAI'99, pages 740--747, San Francisco,
  CA, USA, 1999. Morgan Kaufmann Publishers Inc.

\bibitem{Kearns2002}
M.~Kearns and S.~Singh.
\newblock Near-optimal reinforcement learning in polynomial time.
\newblock {\em Mach. Learn.}, 49(2-3):209--232, Nov. 2002.

\bibitem{ADAM}
D.~P. Kingma and J.~Ba.
\newblock Adam: A method for stochastic optimization, 2014.
\newblock cite arxiv:1412.6980Comment: Published as a conference paper at the
  3rd International Conference for Learning Representations, San Diego, 2015.

\bibitem{VAE}
D.~P. Kingma and M.~Welling.
\newblock Auto-encoding variational bayes, 2013.
\newblock cite arxiv:1312.6114.

\bibitem{Kolter2009}
J.~Z. Kolter and A.~Y. Ng.
\newblock Near-bayesian exploration in polynomial time.
\newblock In {\em Proceedings of the 26th Annual International Conference on
  Machine Learning}, ICML '09, pages 513--520, New York, NY, USA, 2009. ACM.

\bibitem{Lillicrap2016ContinuousCW}
T.~P. Lillicrap, J.~J. Hunt, A.~Pritzel, N.~Heess, T.~Erez, Y.~Tassa,
  D.~Silver, and D.~Wierstra.
\newblock Continuous control with deep reinforcement learning.
\newblock {\em CoRR}, abs/1509.02971, 2016.

\bibitem{luo2018}
Y.~Luo, H.~Xu, Y.~Li, Y.~Tian, T.~Darrell, and T.~Ma.
\newblock Algorithmic framework for model-based deep reinforcement learning
  with theoretical guarantees.
\newblock In {\em International Conference on Learning Representations}, 2019.

\bibitem{mnih16}
V.~Mnih, A.~P. Badia, M.~Mirza, A.~Graves, T.~Lillicrap, T.~Harley, D.~Silver,
  and K.~Kavukcuoglu.
\newblock Asynchronous methods for deep reinforcement learning.
\newblock In M.~F. Balcan and K.~Q. Weinberger, editors, {\em Proceedings of
  The 33rd International Conference on Machine Learning}, volume~48 of {\em
  Proceedings of Machine Learning Research}, pages 1928--1937, New York, New
  York, USA, 20--22 Jun 2016. PMLR.

\bibitem{mnih2015humanlevel}
V.~Mnih, K.~Kavukcuoglu, D.~Silver, A.~A. Rusu, J.~Veness, M.~G. Bellemare,
  A.~Graves, M.~Riedmiller, A.~K. Fidjeland, G.~Ostrovski, S.~Petersen,
  C.~Beattie, A.~Sadik, I.~Antonoglou, H.~King, D.~Kumaran, D.~Wierstra,
  S.~Legg, and D.~Hassabis.
\newblock Human-level control through deep reinforcement learning.
\newblock {\em Nature}, 518(7540):529--533, Feb. 2015.

\bibitem{Moore90}
A.~W. Moore.
\newblock Efficient memory-based learning for robot control.
\newblock Technical report, 1990.

\bibitem{Mueller1997IntegralPM}
A.~Mueller.
\newblock Integral probability metrics and their generating classes of
  functions.
\newblock 1997.

\bibitem{Nagabandi2018}
A.~Nagabandi, G.~Kahn, R.~S.~Fearing, and S.~Levine.
\newblock Neural network dynamics for model-based deep reinforcement learning
  with model-free fine-tuning.
\newblock pages 7559--7566, 05 2018.

\bibitem{VPN}
J.~Oh, S.~Singh, and H.~Lee.
\newblock Value prediction network.
\newblock In I.~Guyon, U.~V. Luxburg, S.~Bengio, H.~Wallach, R.~Fergus,
  S.~Vishwanathan, and R.~Garnett, editors, {\em Advances in Neural Information
  Processing Systems 30}, pages 6118--6128. Curran Associates, Inc., 2017.

\bibitem{BootDQN}
I.~Osband, C.~Blundell, A.~Pritzel, and B.~Van~Roy.
\newblock Deep exploration via bootstrapped dqn.
\newblock In D.~D. Lee, M.~Sugiyama, U.~V. Luxburg, I.~Guyon, and R.~Garnett,
  editors, {\em Advances in Neural Information Processing Systems 29}, pages
  4026--4034. Curran Associates, Inc., 2016.

\bibitem{Osband2017}
I.~Osband, D.~Russo, Z.~Wen, and B.~V. Roy.
\newblock Deep exploration via randomized value functions.
\newblock {\em CoRR}, abs/1703.07608, 2017.

\bibitem{Ostrovski2017}
G.~Ostrovski, M.~G. Bellemare, A.~van~den Oord, and R.~Munos.
\newblock Count-based exploration with neural density models.
\newblock In {\em Proceedings of the 34th International Conference on Machine
  Learning - Volume 70}, ICML'17, pages 2721--2730. JMLR.org, 2017.

\bibitem{Pathak2017}
D.~Pathak, P.~Agrawal, A.~A. Efros, and T.~Darrell.
\newblock Curiosity-driven exploration by self-supervised prediction.
\newblock {\em 2017 IEEE Conference on Computer Vision and Pattern Recognition
  Workshops (CVPRW)}, pages 488--489, 2017.

\bibitem{pathak19disagreement}
D.~Pathak, D.~Gandhi, and A.~Gupta.
\newblock Self-supervised exploration via disagreement.
\newblock In {\em ICML}, 2019.

\bibitem{PER}
T.~Schaul, J.~Quan, I.~Antonoglou, and D.~Silver.
\newblock Prioritized experience replay.
\newblock {\em CoRR}, abs/1511.05952, 2016.

\bibitem{PPO}
J.~Schulman, F.~Wolski, P.~Dhariwal, A.~Radford, and O.~Klimov.
\newblock Proximal policy optimization algorithms.
\newblock {\em CoRR}, abs/1707.06347, 2017.

\bibitem{deeprl}
Z.~Shangtong.
\newblock Modularized implementation of deep rl algorithms in pytorch.
\newblock \url{https://github.com/ShangtongZhang/DeepRL}, 2018.

\bibitem{Shyam2018}
P.~Shyam, W.~Jaskowski, and F.~Gomez.
\newblock Model-based active exploration.
\newblock {\em CoRR}, abs/1810.12162, 2018.

\bibitem{Sorg2010}
J.~Sorg, S.~Singh, and R.~L. Lewis.
\newblock Variance-based rewards for approximate bayesian reinforcement
  learning.
\newblock In {\em Proceedings of the Twenty-Sixth Conference on Uncertainty in
  Artificial Intelligence}, UAI'10, pages 564--571, Arlington, Virginia, United
  States, 2010. AUAI Press.

\bibitem{UPN}
A.~Srinivas, A.~Jabri, P.~Abbeel, S.~Levine, and C.~Finn.
\newblock Universal planning networks: Learning generalizable representations
  for visuomotor control.
\newblock In {\em Proceedings of the 35th International Conference on Machine
  Learning, {ICML} 2018, Stockholmsm{\"{a}}ssan, Stockholm, Sweden, July 10-15,
  2018}, pages 4739--4748, 2018.

\bibitem{Strehl2005}
A.~L. Strehl and M.~L. Littman.
\newblock A theoretical analysis of model-based interval estimation.
\newblock In {\em Proceedings of the 22Nd International Conference on Machine
  Learning}, ICML '05, pages 856--863, New York, NY, USA, 2005. ACM.

\bibitem{Strehl2008}
A.~L. Strehl and M.~L. Littman.
\newblock An analysis of model-based interval estimation for markov decision
  processes.
\newblock {\em J. Comput. Syst. Sci.}, 74(8):1309--1331, Dec. 2008.

\bibitem{Sun2018}
W.~Sun, N.~Jiang, A.~Krishnamurthy, A.~Agarwal, and J.~Langford.
\newblock Model-based reinforcement learning in contextual decision processes.
\newblock {\em CoRR}, abs/1811.08540, 2018.

\bibitem{Acrobot}
R.~S. Sutton.
\newblock Generalization in reinforcement learning: Successful examples using
  sparse coarse coding.
\newblock In {\em Advances in Neural Information Processing Systems 8}, pages
  1038--1044. MIT Press, 1996.

\bibitem{Sutton2008}
R.~S. Sutton, C.~Szepesv\'{a}ri, A.~Geramifard, and M.~Bowling.
\newblock Dyna-style planning with linear function approximation and
  prioritized sweeping.
\newblock In {\em Proceedings of the Twenty-Fourth Conference on Uncertainty in
  Artificial Intelligence}, UAI'08, pages 528--536, Arlington, Virginia, United
  States, 2008. AUAI Press.

\end{thebibliography}

\clearpage

\appendix

\makeatletter
\@addtoreset{theorem}{section}
\@addtoreset{lemma}{section}
\@addtoreset{proposition}{section}
\makeatother

\section{Proofs}
\label{appendix-proofs}

\subsection{Computing Empirical Misfit}

\label{empirical-misfit}

Estimating the misfit $\mathcal{W}(\pi, M, h)$ directly may not be possible when dealing with large state spaces, since a given roll-in state-action pair $(s_{h-1}, a_{h-1})$ may only be observed once and we do not have access to the true model to compute the distribution $P_{M^\star}(\cdot | s_{h-1}, a_{h-1})$. However, we can use an alternate approach based on Integral Probability Metrics (IPMs) \cite{Mueller1997IntegralPM} similar to that described in Appendix B of \cite{Sun2018}.
Let $\mathcal{F} = \{f: \mathcal{S} \times \mathcal{A} \times \mathcal{S} \rightarrow \mathbb{R} : \|f \|_\infty \leq 1\}$.
Using this class of test functions, the total variation distance can be written as:

\begin{align*}
  &\|P_M(\cdot | s_{h-1}, a_{h-1}) - P_{M^\star}(\cdot | s_{h-1}, a_{h-1}) \|_{TV} = \\
  & \max_{f \in \mathcal{F}} \underbrace{\mathbb{E}_{s_h \sim P_M(\cdot | s_{h-1}, a_{h-1})}[f(s_{h-1}, a_{h-1}, s_h)] - \mathbb{E}_{s_h \sim P_{M^\star}(\cdot | s_{h-1}, a_{h-1})}[f(s_{h-1}, a_{h-1}, s_h)]}_{g(M, f, s_{h-1}, a_{h-1})}
\end{align*}

The next lemma shows that $\mathcal{W}(\pi, M, h)$ can be expressed using the IPM definition with the max operator placed outside both expectation operators. This will then allow us to estimate the misfit using a smaller (finite) set of test functions and apply concentration arguments to bound the difference between the true and estimated values.

\begin{techlemma} 
  \begin{align*}
    &\mathcal{W}(\pi, M, h) = \\
    &\max_{f \in \mathcal{F}} \mathbb{E}_{s_{h-1} \sim P_{M^\star}^{\pi, h-1}, a_{h-1} \sim U(\mathcal{A})} \Big[ \mathbb{E}_{s_h \sim P_M(\cdot | s_{h-1}, a_{h-1})}[f(s_{h-1}, a_{h-1}, s_h)] - \mathbb{E}_{s_h \sim P_{M^\star}(\cdot | s_{h-1}, a_{h-1})}[f(s_{h-1}, a_{h-1}, s_h)] \Big]
  \end{align*}
\end{techlemma}

\begin{proof}


Define $f^\mathrm{max}_{s, a, M} = \mbox{argmax}_{f \in \mathcal{F}} g(M, f, s, a)$ and $f^\mathrm{max}_M: \mathcal{S} \times \mathcal{A} \times \mathcal{S} \rightarrow \mathbb{R}$ by $f^\mathrm{max}_M(s, a, s') = f^\mathrm{max}_{s, a, M}(s, a, s')$. Note that $\| f^\mathrm{max}_M \|_\infty \leq 1$ so $f^\mathrm{max}_M \in \mathcal{F}$.

We can then write:

\begin{align*}
  \mathcal{W}(\pi, M, h)
  &= \mathbb{E}_{s_{h-1} \sim P_{M^\star}^{\pi, h-1}, a_{h-1} \sim U(\mathcal{A})}\big[\max_{f \in \mathcal{F}} g(M, f, s_{h-1}, a_{h-1})\big] \\
  &= \mathbb{E}_{s_{h-1} \sim P_{M^\star}^{\pi, h-1}, a_{h-1} \sim U(\mathcal{A})}\big[g(M, f^\mathrm{max}_{s_{h-1}, a_{h-1}, M}, s_{h-1}, a_{h-1})\big] \\
  &= \mathbb{E}_{s_{h-1} \sim P_{M^\star}^{\pi, h-1}, a_{h-1} \sim U(\mathcal{A})}\big[g(M, f^\mathrm{max}_M, s_{h-1}, a_{h-1})\big] \\
  &\leq \max_{f \in \mathcal{F}} \mathbb{E}_{s_{h-1} \sim P_{M^\star}^{\pi, h-1}, a_{h-1} \sim U(\mathcal{A})}\big[g(M, f, s_{h-1}, a_{h-1})\big] \\
\end{align*}

Now let $f^\star$ be the function which maximizes the last quantity. We then have:

\begin{align*}
  &\max_{f \in \mathcal{F}} \mathbb{E}_{s_{h-1} \sim P_{M^\star}^{\pi, h-1}, a_{h-1} \sim U(\mathcal{A})}\big[g(M, f, s_{h-1}, a_{h-1})\big] \\
  &=\mathbb{E}_{s_{h-1} \sim P_{M^\star}^{\pi, h-1}, a_{h-1} \sim U(\mathcal{A})}\big[g(M, f^\star, s_{h-1}, a_{h-1})\big] \\
  &\leq\mathbb{E}_{s_{h-1} \sim P_{M^\star}^{\pi, h-1}, a_{h-1} \sim U(\mathcal{A})}\big[\max_{f \in \mathcal{F}} g(M, f, s_{h-1}, a_{h-1})\big] \\
  &= \mathcal{W}(\pi, M, h)
\end{align*}

Combining the two inequalities gives the result.

\end{proof}

We next define a new set of functions $\widetilde{\mathcal{F}}$ as follows. Let
\begin{equation*}
  f_{\pi, M, h} = \mbox{argmax}_{f \in \mathcal{F}} \mathbb{E}_{s_{h-1} \sim P_{M^\star}^{\pi, h-1}, a_{h-1} \sim U(\mathcal{A})} \Big[ g(M, f, s_{h-1}, a_{h-1}) \Big]
\end{equation*}

and $\widetilde{\mathcal{F}} = \{ \pm f_{\pi, M, h}: \pi \in \Pi, M \in \mathcal{M}, h \in [H]\}$. We then have:

\begin{align*}
  \mathcal{W}(\pi, M, h) &= \max_{f \in \mathcal{F}} \mathbb{E}_{s_{h-1} \sim P_{M^\star}^{\pi, h-1}, a_{h-1} \sim U(\mathcal{A})} \Big[ g(M, f, s_{h-1}, a_{h-1}) \Big] \\
  &= \max_{f \in \widetilde{\mathcal{F}}} \mathbb{E}_{s_{h-1} \sim P_{M^\star}^{\pi, h-1}, a_{h-1} \sim U(\mathcal{A})} \Big[ g(M, f, s_{h-1}, a_{h-1}) \Big]
\end{align*}

The misfit can thus be computed using a smaller (finite) set of test functions $\widetilde{\mathcal{F}}$, with size $|\widetilde{\mathcal{F}}| \leq |\Pi|\cdot |\mathcal{M}|\cdot H$.

Given a dataset $\mathcal{R}_\pi = \{(s_{h-1}^{(i)}, a_{h-1}^{(i)}, s_h^{(i)} \}_{n=1}^n$ generated by following policy $\pi$, we estimate the empirical misfit for a model $M$ at time step $h$ using $\widetilde{\mathcal{F}}$ as follows:

\begin{align*}
  &\widetilde{\mathcal{W}}(\pi, M, h) = \max_{f \in \widetilde{\mathcal{F}}} \frac{1}{n} \sum_{i=1}^n \Big[ \mathbb{E}_{s_h \sim P_M(\cdot | s_{h-1}, a_{h-1})}[f(s_{h-1}, a_{h-1}, s_h)] - f(s_{h-1}^{(i)}, a_{h-1}^{(i)}, s_h^{(i)})] \Big]
\end{align*}

\begin{techlemma}
  \label{misfit-deviation}
  (Deviation Bound for $\widetilde{\mathcal{W}}(\pi, M, h)$).
  Fix $h$ and $\pi \in \Pi$. Sample a dataset $\Big\{(s_{h-1}^{(i)}, a_{h-1}^{(i)}, s_{h}^{(i)})\Big\}_{i=1}^n$ of size $n$ with:
  \begin{equation*}
    s_{h-1}^{(i)} \sim P_{M^\star}^{\pi, h-1}, a_{h-1}^{(i)} \sim U(\mathcal{A}), s_{h}^{(i)} \sim P_{M^\star}(\cdot | s_{h-1}^{(i)}, a_{h-1}^{(i)})
  \end{equation*}

  Then with probability at least $1-\delta$, we have for all $M \in \mathcal{M}$:
  \begin{equation*}
    \Big| \widetilde{\mathcal{W}}(\pi, M, h) - \mathcal{W}(\pi, M, h)  \Big| \leq \frac{4\log(2|\mathcal{M}||\Pi|H/\delta)}{3n}  + 4 \sqrt{\frac{2\log(2|\mathcal{M}||\Pi|H/\delta)}{n}}
  \end{equation*}
\end{techlemma}

\begin{proof}
  Fix $M \in \mathcal{M}$ and $f \in \widetilde{\mathcal{F}}$. Define the random variable $z_i(M, f)$ as

  \begin{equation*}
    z_i(M, f) = \mathbb{E}_{s_h \sim P_M(\cdot | s_{h-1}^{(i)}, a_{h-1}^{(i)})} f(s_{h-1}^{(i)}, a_{h-1}^{(i)}, s_h) - f(s_{h-1}^{(i)}, a_{h-1}^{(i)}, s_h^{(i)})
  \end{equation*}

  The expectation is given by:

  \begin{equation*}
    \mathbb{E}[z_i(M, f)] = \mathbb{E}_{s_{h-1} \sim P_{M^\star}^{\pi, h}, a_{h-1} \sim U} \Big[\mathbb{E}_{s_h \sim P_M(\cdot | s_{h-1}, a_{h-1})}[f(s_{h-1}, a_{h-1}, s_h)] - \mathbb{E}_{s_h \sim P_{M^\star}(\cdot | s_{h-1}, a_{h-1})}[f(s_{h-1}, a_{h-1}, s_h)] \Big]
  \end{equation*}

  Note that $|z_i(M, f)| \leq 2$ and $\mathrm{Var}(z_i(M, f)) \leq 2$. Therefore we can apply Bernstein's inequality which states that for any $\epsilon$:

  \begin{align*}
    P\Big[ \Big| \sum_{i=1}^n (z_i(M, f) - \mathbb{E}[z_i(M, f)) \Big| > \epsilon \Big] &\leq 2\mathrm{exp} \Big( - \frac{\epsilon^2/2}{\sum_{i=1}^n \mathbb{E}[(z_i(M, f) - \mathbb{E}[z_i(M, f)])^2] + 2\epsilon/3}  \Big) \\
      &\leq 2\mathrm{exp} \Big( - \frac{\epsilon^2/2}{2n + 2\epsilon/3}  \Big) \triangleq \delta \\
  \end{align*}

  Solving for $\epsilon$ in terms of $\delta$, we get: $\frac{\epsilon^2/2}{2n + 2\epsilon/3} = \log(2/\delta) \implies \epsilon^2 - 4n \log(2/\delta) - \frac{2}{3} \epsilon \log(2/\delta) = 0$. Applying the quadratic formula then gives us:

  \begin{align*}
    \epsilon &= \frac{1}{3} \log(2/\delta) + \frac{1}{2} \sqrt{(\frac{2}{3} \log(2/\delta))^2 + 16n \log(2/\delta)} \\
     &\leq \frac{1}{3} \log(2/\delta) + \frac{1}{2} \sqrt{(\frac{2}{3} \log(2/\delta))^2} + \sqrt{16n \log(2/\delta)} \\
     &= \frac{2}{3} \log(2/\delta) + 4 \sqrt{n \log(2/\delta)} \\
  \end{align*}

  Therefore with probability at least $1-\delta$ we have:

  \begin{align*}
    \Big| \sum_{i=1}^n (z_i(M, f) - \mathbb{E}[z_i(M, f)) \Big| < \epsilon \leq \frac{2}{3} \log(2/\delta) + 4 \sqrt{n \log(2/\delta)} \\
  \end{align*}

  And therefore:

  \begin{align*}
    \Big| \Big[ \frac{1}{n} \sum_{i=1}^n (z_i(M, f) \Big] - \mathbb{E}[z_i(M, f)] \Big| \leq \frac{2\log(2/\delta)}{3n}  + 4 \frac{\sqrt{n \log(2/\delta)}}{n} = \frac{2\log(2/\delta)}{3n}  + 4 \sqrt{\frac{\log(2/\delta)}{n}} \\
  \end{align*}

  Via a union bound over $\mathcal{M}$ and $\widetilde{\mathcal{F}}$, we have that for all pairs $M \in \mathcal{F}$ and $f \in \widetilde{\mathcal{F}}$, with probability at least $1-\delta$:
  
  \begin{align*}
    \Big| \Big[ \frac{1}{n} \sum_{i=1}^n (z_i(M, f) \Big] - \mathbb{E}[z_i(M, f)] \Big|
    &\leq \frac{2\log(2|\mathcal{M}||\widetilde{\mathcal{F}}|/\delta)}{3n}  + 4 \sqrt{\frac{\log(2|\mathcal{M}||\widetilde{\mathcal{F}}|/\delta)}{n}} \\
    &\leq \frac{2\log(2|\mathcal{M}|^2|\Pi|H/\delta)}{3n}  + 4 \sqrt{\frac{\log(2|\mathcal{M}|^2|\Pi|H/\delta)}{n}} \\
    &\leq \frac{4\log(2|\mathcal{M}||\Pi|H/\delta)}{3n}  + 4 \sqrt{\frac{2\log(2|\mathcal{M}||\Pi|H/\delta)}{n}} \\
  \end{align*}

  Note that $\mathcal{W}(\pi, M, h) = \max_{f \in \widetilde{\mathcal{F}}} \mathbb{E}[z_i(M, f)]$ and $\widetilde{\mathcal{W}}(\pi, M, h) = \max_{f \in \widetilde{\mathcal{F}}} \frac{1}{n}\sum_{i=1}^n z_i(M, f)$. For a fixed $M$, we have shown uniform convergence over $\widetilde{\mathcal{F}}$ which implies that the empirical and population maxima must be similarly close, which yields the result.
  
\end{proof}


\subsection{Main Results}

\begin{proposition}
Assume $|\mathcal{S}|$ is finite and let $A_h$ be the matrix defined above. Then $rank(A_h) \leq |\mathcal{S}|$.
\end{proposition}
\begin{proof}
This proposition is a special case of Proposition 2 so the proof carries over. It can also be shown with a direct argument as follows.
Define the matrix $U_h \in \mathbb{R}^{|\Pi| \times |\mathcal{S}|}$ by $U_h(\pi, s) = P_{M^\star}^{\pi, h-1}(s)$
and the matrix $V_h \in \mathbb{R}^{|\mathcal{M}| \times |\mathcal{S}|}$ by $V_h(M, s) = \mathbb{E}_{a \sim U(\mathcal{A})}[\|P_{M}(\cdot | s, a) - P_{M^\star}(\cdot | s, a) \|_{TV}]$.

Then we can write:
\begin{align*}
  A_h(\pi, M) = \mathcal{W}(\pi, M, h) &= \mathbb{E}_{s \sim P_{M^\star}^{\pi, h-1}, a \sim U(\mathcal{A})} \big[ ||P_M(\cdot | s, a) - P_{M^\star}(\cdot | s, a) ||_{TV} \big] \\
  &= \sum_s P_{M^\star}^{\pi, h-1}(s) \mathbb{E}_{a \sim U(\mathcal{A})}[||P_M(\cdot | s, a) - P_{M^\star}(\cdot | s, a) ||_{TV}] \\
  &= \sum_s U_h(\pi, s) V_h(M, s) \\
\end{align*}

Therefore we have $A_h = U_hV_h^\top$ and so $rank(A_h) \leq |\mathcal{S}|$.
\end{proof}

\begin{proposition}
  Let $\Gamma$ denote the true transition matrix of size $|\mathcal{S}| \times |\mathcal{S} \times \mathcal{A}|$, with $\Gamma(s', (s, a)) = P_{M^\star}(s'|s, a)$.
  Assume that there exist two matrices $\Gamma_1, \Gamma_2$ with sizes $|\mathcal{S}| \times K$ and $K \times |\mathcal{S} \times \mathcal{A}|$ such that $\Gamma=\Gamma_1\Gamma_2$.
  Then $rank(A_h) \leq K$.
\end{proposition}

\begin{proof}
  We first define the vectors $z^{\pi, h}$ of size $K$ as follows:
  \begin{align*}
    z^{\pi, h}_k &= \sum_{s_{h-1}} \sum_{a_{h-1}} P_{M^\star}^{\pi, h-1}(s_{h-1}) \pi(a_{h-1}|s_{h-1}) \Gamma_2(k , (s_{h-1}, a_{h-1})) \\
  \end{align*}
  This allows us to rewrite:
  \begin{align*}
    P_{M^\star}^{\pi, h}(s_h) &= \sum_{s_{h-1}} \sum_{a_{h-1}} P_{M^\star}^{\pi, h-1}(s_{h-1}) \pi(a_{h-1}|s_{h-1})P_{M^\star}(s_h | s_{h-1}, a_{h-1}) \\
    &= \sum_{s_{h-1}} \sum_{a_{h-1}}\sum_{k=1}^K P_{M^\star}^{\pi, h-1}(s_{h-1}) \pi(a_{h-1}|s_{h-1}) \Gamma_2(k , (s_{h-1}, a_{h-1})) \Gamma_1(s_h , k) \\
    &= \sum_{k=1}^K z^{\pi, h}_k \Gamma_1(s_h, k)
  \end{align*}

  We can now rewrite the witnessed model misfit as follows:

  \begin{align*}
    \mathcal{W}(\pi, M, h) &= \mathbb{E}_{s_{h-1} \sim P_{M^\star}^{\pi, h-1}, a_{h-1} \sim U(\mathcal{A})} \big[ \|P_M(\cdot | s_{h-1}, a_{h-1}) - P_{M^\star}(\cdot | s_{h-1}, a_{h-1}) \|_{TV}\big] \\
    &= \sum_{s_{h-1}} \sum_{a_{h-1}} P_{M^\star}^{\pi, h-1}(s_{h-1}) \frac{1}{|\mathcal{A}|} \big[ \|P_M(\cdot | s_{h-1}, a_{h-1}) - P_{M^\star}(\cdot | s_{h-1}, a_{h-1}) \|_{TV}\big] \\
    &= \sum_{s_{h-1}} \sum_{a_{h-1}} \sum_{k=1}^K z^{\pi, h-1}_k \Gamma_1(s_{h-1}, k) \frac{1}{|\mathcal{A}|} \big[ \|P_M(\cdot | s_{h-1}, a_{h-1}) - P_{M^\star}(\cdot | s_{h-1}, a_{h-1}) \|_{TV}\big] \\
    &= \sum_{k=1}^K z^{\pi, h-1}_k \sum_{s_{h-1}} \sum_{a_{h-1}}  \Gamma_1(s_{h-1}, k) \frac{1}{|\mathcal{A}|} \big[ \|P_M(\cdot | s_{h-1}, a_{h-1}) - P_{M^\star}(\cdot | s_{h-1}, a_{h-1}) \|_{TV}\big] \\
  \end{align*}

  Define the matrices $U_h$ and $V_h$ of size $|\Pi| \times K$ and $|\mathcal{M}| \times K$ by:
  \begin{align*}
    U_h(\pi, k) &= z^{\pi, h-1}_k \\
    V_h(M, k) &= \sum_{s_{h-1}} \sum_{a_{h-1}}  \Gamma_1(s_{h-1}, k) \frac{1}{|\mathcal{A}|} \big[ \|P_M(\cdot | s_{h-1}, a_{h-1}) - P_{M^\star}(\cdot | s_{h-1}, a_{h-1}) \|_{TV}\big]
  \end{align*}

  We then have $A_h = U_hV_h^\top$, which proves the desired result.
\end{proof}

\begin{lemma}
  Let $\mathcal{M}$ be a set of models and $\Pi$ a set of policies.
  If there exist $M, M' \in \mathcal{M}, \pi \in \Pi$ and $h \leq H$ such that $\mathcal{D}(\pi, M, M', h) > \alpha$, then there exists $h' \leq h$ such that $\mathcal{W}(\pi, M, h') > \frac{\alpha}{4|\mathcal{A}|\cdot H}$ or $\mathcal{W}(\pi, M', h') > \frac{\alpha}{4|\mathcal{A}|\cdot H}$ (or both).
\end{lemma}


\begin{proof}
  If there exists $h' \leq h-1$ such that either $\mathcal{W}(\pi, M, h') > \frac{\alpha}{4|\mathcal{A}|\cdot H}$ or $\mathcal{W}(\pi, M', h') > \frac{\alpha}{4|\mathcal{A}|\cdot H}$ then we are done.
  Therefore assume that $\mathcal{W}(\pi, M, h'), \mathcal{W}(\pi, M', h') \leq \frac{\alpha}{4|\mathcal{A}|\cdot H}$ for all $h' \in [H-1]$.
  To keep notation light, for the following we will use the following abbreviations:
  \begin{align*}
    P_M^h &:= P_M(s_h | s_{h-1}, a_{h-1}) \\
    P_M^{\pi, h-1} &:= P_M^{\pi, h-1}(s_{h-1}) \\
  \end{align*}

  We now write:

  \begin{align*}
    &\mathcal{D}(\pi, M, M', h) \\
    &= \frac{1}{|\mathcal{A}|}\sum_{s_{h-1}} \sum_{a_{h-1}} \sum_{s_h} |P_{M}^h P_{M}^{\pi, h-1} - P_{M'}^h P_{M'}^{\pi, h-1}| \\
    &=\frac{1}{|\mathcal{A}|}\sum_{s_{h-1}} \sum_{a_{h-1}} \sum_{s_h} |P_{M}^h P_{M}^{\pi, h-1} - P_{M'}^h P_{M'}^{\pi, h-1} - P_{M}^h P_{M^\star}^{\pi, h-1} + P_{M}^h P_{M^\star}^{\pi, h-1} + P_{M'}^h P_{M^\star}^{\pi, h-1} - P_{M'}^h P_{M^\star}^{\pi, h-1} | \\
    &=\frac{1}{|\mathcal{A}|}\sum_{s_{h-1}} \sum_{a_{h-1}} \sum_{s_h} |P_{M}^h(P_{M}^{\pi, h-1} - P_{M^\star}^{\pi, h-1}) + P_{M'}^h (P_{M^\star}^{\pi, h-1} - P_{M'}^{\pi, h-1}) + (P_{M}^h - P_{M'}^h) P_{M^\star}^{\pi, h-1} | \\
    &\leq\frac{1}{|\mathcal{A}|}\sum_{s_{h-1}} \sum_{a_{h-1}} \sum_{s_h} P_{M}^h|P_{M}^{\pi, h-1} - P_{M^\star}^{\pi, h-1}| + P_{M'}^h |P_{M^\star}^{\pi, h-1} - P_{M'}^{\pi, h-1}| + |P_{M}^h - P_{M'}^h| P_{M^\star}^{\pi, h-1} \\
    &\leq\frac{1}{|\mathcal{A}|}\sum_{s_{h-1}} \sum_{a_{h-1}} \sum_{s_h} P_{M}^h|P_{M}^{\pi, h-1} - P_{M^\star}^{\pi, h-1}| + P_{M'}^h|P_{M^\star}^{\pi, h-1} - P_{M'}^{\pi, h-1}| + |P_{M}^h - P_{M^\star}^h| P_{M^\star}^{\pi, h-1} + |P_{M'}^h - P_{M^\star}^h| P_{M^\star}^{\pi, h-1} \\
    &=\frac{1}{|\mathcal{A}|}\sum_{s_{h-1}} \sum_{a_{h-1}} \sum_{s_h} P_{M}^h|P_{M}^{\pi, h-1} - P_{M^\star}^{\pi, h-1}| + \frac{1}{|\mathcal{A}|}\sum_{s_{h-1}} \sum_{a_{h-1}} \sum_{s_h} P_{M'}^h |P_{M^\star}^{\pi, h-1} - P_{M'}^{\pi, h-1}|  \\
    & \phantom{========} + \frac{1}{|\mathcal{A}|}\sum_{s_{h-1}} \sum_{a_{h-1}} \sum_{s_h} |P_{M}^h - P_{M^\star}^h| P_{M^\star}^{\pi, h-1} + \frac{1}{|\mathcal{A}|}\sum_{s_{h-1}} \sum_{a_{h-1}} \sum_{s_h} |P_{M'}^h - P_{M^\star}^h| P_{M^\star}^{\pi, h-1} \\
    &=\frac{1}{|\mathcal{A}|}\sum_{s_{h-1}} \sum_{a_{h-1}} \sum_{s_h} P_{M}^h|P_{M}^{\pi, h-1} - P_{M^\star}^{\pi, h-1}| + \frac{1}{|\mathcal{A}|}\sum_{s_{h-1}} \sum_{a_{h-1}} \sum_{s_h} P_{M'}^h |P_{M^\star}^{\pi, h-1} - P_{M'}^{\pi, h-1}| + \mathcal{W}(\pi, M, h) + \mathcal{W}(\pi, M', h)
  \end{align*}

  We now bound the first term in this sum:

  \begin{align*}
    &\frac{1}{|\mathcal{A}|} \sum_{s_{h-1}} \sum_{a_{h-1}}\sum_{s_h} P_M^h\Big|P_M^{\pi, h-1} - P_{M^\star}^{\pi, h-1}\Big| \\
    &= \sum_{s_{h-1}} \Big|P_M^{\pi, h-1} - P_{M^\star}^{\pi, h-1}\Big| \\
    &= \sum_{s_{h-1}} \Big| \sum_{s_{h-2}} \sum_{a_{h-2}} \pi(a_{h-2} | s_{h-2}) (P_M^{h-1}P_M^{\pi, h-2} - P_{M^\star}^{h-1}P_{M^\star}^{\pi, h-2} ) \Big| \\
    &= \sum_{s_{h-1}} \sum_{s_{h-2}} \sum_{a_{h-2}} \pi(a_{h-2} | s_{h-2}) \Big|P_M^{h-1}P_M^{\pi, h-2} - P_{M^\star}^{h-1}P_{M^\star}^{\pi, h-2} - P_M^{h-1}P_{M^\star}^{\pi, h-2} + P_M^{h-1}P_{M^\star}^{\pi, h-2}\Big|  \\
    &= \sum_{s_{h-1}} \sum_{s_{h-2}} \sum_{a_{h-2}} \pi(a_{h-2} | s_{h-2}) \Big|P_M^{h-1}(P_M^{\pi, h-2} - P_{M^\star}^{\pi, h-2} )  + (P_M^{h-1}- P_{M^\star}^{h-1}) P_{M^\star}^{\pi, h-2} \Big| \\
    &\leq \sum_{s_{h-1}} \sum_{s_{h-2}} \sum_{a_{h-2}} \pi(a_{h-2} | s_{h-2}) \Big|P_M^{h-1}(P_M^{\pi, h-2} - P_{M^\star}^{\pi, h-2} ) \Big | + \sum_{s_{h-1}} \sum_{s_{h-2}} \sum_{a_{h-2}} \pi(a_{h-2}| s_{h-2}) \Big|(P_M^{h-1}- P_{M^\star}^{h-1}) P_{M^\star}^{\pi, h-2} \Big| \\
    &\leq \sum_{s_{h-1}} \sum_{s_{h-2}} \sum_{a_{h-2}} \pi(a_{h-2} | s_{h-2}) \Big|P_M^{h-1}(P_M^{\pi, h-2} - P_{M^\star}^{\pi, h-2} ) \Big | + \sum_{s_{h-1}} \sum_{s_{h-2}} \sum_{a_{h-2}} \Big|(P_M^{h-1}- P_{M^\star}^{h-1}) P_{M^\star}^{\pi, h-2} \Big| \\
    &= \sum_{s_{h-2}} \Big[\sum_{s_{h-1}} \Big[\sum_{a_{h-2}} \pi(a_{h-2} | s_{h-2}) \Big] P_M^{h-1} \Big] \Big|P_M^{\pi, h-2} - P_{M^\star}^{\pi, h-2} \Big| + |\mathcal{A}| \cdot \mathcal{W}(\pi, M, h-1) \\
    &= \sum_{s_{h-2}} \Big|P_M^{\pi, h-2} - P_{M^\star}^{\pi, h-2} \Big| + |\mathcal{A}| \cdot \mathcal{W}(\pi, M, h-1) \\
    &\leq \sum_{s_{h-2}} \Big|P_M^{\pi, h-2} - P_{M^\star}^{\pi, h-2} \Big| + |\mathcal{A}| \cdot \frac{\alpha}{4|\mathcal{A}|\cdot H} \\
    &= \sum_{s_{h-2}} \Big|P_M^{\pi, h-2} - P_{M^\star}^{\pi, h-2} \Big| + \frac{\alpha}{4H} \\
  \end{align*}

  By induction on $h$, we have

  \begin{align*}
    \frac{1}{|\mathcal{A}|}\sum_{s_{h-1}} \sum_{a_{h-1}} \sum_{s_h} P_{M}^h|P_{M}^{\pi, h-1} - P_{M^\star}^{\pi, h-1}| = \sum_{s_{h-1}} \Big|P_M^{\pi, h-1} - P_{M^\star}^{\pi, h-1}\Big| \leq h \cdot \frac{\alpha}{4H} \leq \frac{\alpha}{4}  \\
  \end{align*}

  An analogous argument shows that

  \begin{align*}
    \frac{1}{|\mathcal{A}|}\sum_{s_{h-1}} \sum_{a_{h-1}} \sum_{s_h} |P_{M'}^h(P_{M'}^{\pi, h-1} - P_{M^\star}^{\pi, h-1})| \leq \frac{\alpha}{4} \\
  \end{align*}

  Putting these together we have:

  \begin{align*}
    \alpha \leq \mathcal{D}(\pi, M, M', h) &\leq \frac{1}{|\mathcal{A}|}\sum_{s_{h-1}} \sum_{a_{h-1}} \sum_{s_h} P_{M}^h|P_{M}^{\pi, h-1} - P_{M^\star}^{\pi, h-1}| \\
    & + \frac{1}{|\mathcal{A}|}\sum_{s_{h-1}} \sum_{a_{h-1}} \sum_{s_h} P_{M'}^h |P_{M^\star}^{\pi, h-1} - P_{M'}^{\pi, h-1}| \mathcal{W}(\pi, M, h) + \mathcal{W}(\pi, M', h) \\
    &\leq \frac{\alpha}{4} + \frac{\alpha}{4} + \mathcal{W}(\pi, M, h) + \mathcal{W}(\pi, M', h) \\
    &= \frac{\alpha}{2} + \mathcal{W}(\pi, M, h) + \mathcal{W}(\pi, M', h)
  \end{align*}

  Therefore $\mathcal{W}(\pi, M, h) + \mathcal{W}(\pi, M', h) \geq \alpha/2$ and since $\mathcal{W}(\pi, M, h), \mathcal{W}(\pi, M', h) \geq 0$ we have either $\mathcal{W}(\pi, M, h) \geq \frac{\alpha}{4} \geq \frac{\alpha}{4|\mathcal{A}|H}$ or $\mathcal{W}(\pi, M', h) \geq \frac{\alpha}{4} \geq \frac{\alpha}{4|\mathcal{A}|H}$, as desired.
\end{proof}

\begin{lemma}
  \label{terminate-explore}
  (Explore or Exploit)
  Suppose the true model $M^\star$ is never eliminated.
  At iteration $t$, one of the following two conditions must hold: either there exists $M \in \mathcal{M}_{t}, h_t \leq H$ such that $\mathcal{W}(\pi_{\mathrm{explore}}^t, M, h_t) > \frac{\epsilon}{4H^2|\mathcal{A}|^2}$, or the algorithm returns $\pi_\mathrm{exploit}$ such that $v_{\pi_{\mathrm{exploit}}} > v_{\pi^\star} - \epsilon$.
\end{lemma}

\begin{proof}
  First consider the case where $v_{\mathrm{explore}}(\pi_\mathrm{explore}^t, \mathcal{M}_t) > \frac{\epsilon}{|\mathcal{A}|}$. Then by definition of $v_{\mathrm{explore}}$ there exists some $M, M'$ and $h \in [H]$ such that $\mathcal{D}(\pi_\mathrm{explore}^t, M, M', h) > \frac{\epsilon}{H|\mathcal{A}|}$.
  By Lemma 1 we also have $\mathcal{W}(\pi_\mathrm{explore}^t, M, h_t) > \frac{\epsilon}{4H^2|\mathcal{A}|^2}$ or $\mathcal{W}(\pi_\mathrm{explore}^t, M', h_t) > \frac{\epsilon}{4H^2|\mathcal{A}|^2}$ for some $h_t \leq h$.

  Now consider the case where $v_{\mathrm{explore}}(\pi_\mathrm{explore}^t, \mathcal{M}_t) \leq \frac{\epsilon}{|\mathcal{A}|}$.
  Since $\pi_{\mathrm{exploit}}$ is the optimal policy for $\tilde{M}$, we have $v_{\tilde{M}}^{\pi_{\mathrm{exploit}}} \geq v_{\tilde{M}}^{\pi^\star}$.

  We will now bound $|v_{\tilde{M}}^{\pi^\star} - v_{M^\star}^{\pi^\star}|$ :

  \begin{align*}
    |v_{\tilde{M}}^{\pi^\star} - v_{M^\star}^{\pi^\star}| &= \Big| \sum_{h=1}^H \sum_{s_h} P_{\tilde{M}}^{\pi^\star, h}(s_h) R^\star(s_h) - \sum_{h=1}^H \sum_{s_h} P_{M^\star}^{\pi^\star, h}(s_h) R^\star(s_h) \Big|  \\
    &= \Big| \sum_{h=1}^H \sum_{s_h} (P_{\tilde{M}}^{\pi^\star, h}(s_h) - P_{M^\star}^{\pi^\star, h}(s_h)) R^\star(s_h) \Big| \\
    &\leq \sum_{h=1}^H \sum_{s_h} |P_{\tilde{M}}^{\pi^\star, h}(s_h) - P_{M^\star}^{\pi^\star, h}(s_h)| \\
  \end{align*}

  where we have used the fact that the per-timestep rewards are bounded by 1. Expanding further we get:

  \begin{align*}
    |v_{\tilde{M}}^{\pi^\star} - v_{M^\star}^{\pi^\star}|
    &\leq  \sum_{h=1}^H \sum_{s_h} |P_{\tilde{M}}^{\pi^\star, h} - P_{M^\star}^{\pi^\star, h}(s_h)| \\
    &= \sum_{h=1}^H \sum_{s_h} \Big|\sum_{s_{h-1}}\sum_{a_{h-1}} P_{\tilde{M}}^{\pi^\star, h-1}\pi^\star(a_{h-1}|s_{h-1})P_{\tilde{M}}^h - P_{M^\star}^{\pi^\star, h-1}\pi^\star(a_{h-1}|s_{h-1})P_{M^\star}^h \Big| \\
    &\leq \sum_{h=1}^H \sum_{s_h} \sum_{s_{h-1}}\sum_{a_{h-1}} \Big| P_{\tilde{M}}^{\pi^\star, h-1}\pi^\star(a_{h-1}|s_{h-1})P_{\tilde{M}}^h - P_{M^\star}^{\pi^\star, h-1}\pi^\star(a_{h-1}|s_{h-1})P_{M^\star}^h \Big| \\
   &\leq \sum_{h=1}^H \sum_{s_h} \sum_{s_{h-1}} \Big| P_{\tilde{M}}^{\pi^\star, h-1}P_{\tilde{M}}^h - P_{M^\star}^{\pi^\star, h-1}P_{M^\star}^h \Big| \cdot |\mathcal{A}| \\
   &\leq \sum_{h=1}^H \mathcal{D}(\pi^\star, \tilde{M}, M^\star, h) |\mathcal{A}| \\
  \end{align*}

  Note that $\sum_{h=1}^H \mathcal{D}(\pi^\star, \tilde{M}, M^\star, h) \leq v_{\mathrm{explore}}(\pi^\star, \mathcal{M}_t) \leq v_{\mathrm{explore}}(\pi_\mathrm{explore}^t, \mathcal{M}_t) \leq \frac{\epsilon}{|\mathcal{A}|}$ since $\pi_\mathrm{explore}^t$ is the optimal policy for the exploration MDP. Therefore we have

  \begin{align*}
      |v_{\tilde{M}}^{\pi^\star} - v_{M^\star}^{\pi^\star}|  &\leq \frac{\epsilon}{|\mathcal{A}|} \cdot |\mathcal{A}| = \epsilon
\end{align*}





  Combining this with the fact that $v_{\tilde{M}}^{\pi_{\mathrm{exploit}}} \geq v_{\tilde{M}}^{\pi^\star}$, we get $v_{M^\star}^{\pi_{\mathrm{exploit}}} \geq v_{M^\star}^{\pi_{\mathrm{exploit}}} - \epsilon$.

\end{proof}

The proof for the following lemma can be found in \cite{Sun2018} (Lemma 8).
\begin{techlemma}
  Suppose that $|\widetilde{\mathcal{W}}(\pi_\mathrm{explore}^t, M, h_t) - \mathcal{W}(\pi_\mathrm{explore}^t, M, h_t)| \leq \phi$ holds for all $t, h_t$ and $M \in \mathcal{M}$. Then:
  \begin{enumerate}
  \item $M^\star \in \mathcal{M}_t$ for all $t$.
    \item Denote $\widetilde{\mathcal{M}}_{t+1} = \{ M \in \widetilde{\mathcal{M}}_{t}: A_{h_t}(\pi_\mathrm{explore}^t, M) \leq 2\phi \}$ with $\widetilde{\mathcal{M}}_1 = \mathcal{M}$. We have $\mathcal{M}_t \subseteq \widetilde{\mathcal{M}}_t$ for all $t$.
  \end{enumerate}
\end{techlemma}

\begin{lemma}
  \label{iteration-complexity}
  (Iteration Complexity)
  Let $d=\max_{1 \leq h \leq H} rank(A_h)$ and $\phi=\frac{\epsilon}{24H^2|\mathcal{A}|^2\sqrt{d}}$. Suppose that $|\widetilde{\mathcal{W}}(\pi_\mathrm{explore}^t, M, h) - \mathcal{W}(\pi_\mathrm{explore}^t, M, h)| \leq \phi$ holds for all $t$, $h \leq H$ and $M \in \mathcal{M}$. Then the number of rounds of Algorithm \ref{alg:main} with the \texttt{UpdateModelSet} routine given by Algorithm \ref{alg:update-version-space} is at most $Hd \log(\frac{\beta}{2\phi}) / \log(5/3)$.
\end{lemma}

\begin{proof}
  From Lemma \ref{terminate-explore}, if the algorithm does not terminate then we have $\pi_\mathrm{explore}^t, h_t, M' \in \mathcal{M}_{t}$ such that:

  \begin{align*}
    \mathcal{W}(\pi_\mathrm{explore}^t, M', h_t) > \frac{\epsilon}{4H^2|\mathcal{A}|^2} = 6\sqrt{d}\phi
  \end{align*}

  which can be rewritten as:

  \begin{align*}
    A_{h_t}(\pi_\mathrm{explore}^t, M') = U_{h_t}(\pi_\mathrm{explore}^t)^\top V_{h_t}(M') > 6\sqrt{d}\phi.
  \end{align*}

  For any $h$ and $t$, denote $O_{t}^h$ as the origin-centered minimum volume enclosing ellipsoid (MVEE) of $\{V_h(M): M \in \widetilde{\mathcal{M}}_{t}\}$.
  Also denote $O_{t,+}^{h_t}$ as the origin-centered MVEE of $\{v \in O_{t}^{h_t}: U_{h_t}(\pi_\mathrm{explore}^t))^\top v \leq 2\phi \}$.
  Note that by definition of $\widetilde{\mathcal{M}}_{t+1}$, for all $M \in \widetilde{\mathcal{M}}_{t+1}$ we have $A_{h_t}(\pi_\mathrm{explore}^t, M) = U_{h_t}(\pi_\mathrm{explore}^t)^\top V_{h_t}(M) \leq 2\phi$ and since $O_{t+1}^{h_t} \subseteq O_{t}^{h_t}$ we have $O_{t+1}^{h_t} \subseteq O_{t,+}^{h_t}$ and hence $\mbox{vol}(O_t^{h_{t}}) \leq \mbox{vol}(O_{t, +}^{h_t})$. See Figure \ref{fig:ellipsoids} for an illustration.

\begin{figure}
  \centering
  \includegraphics[width=0.7\textwidth]{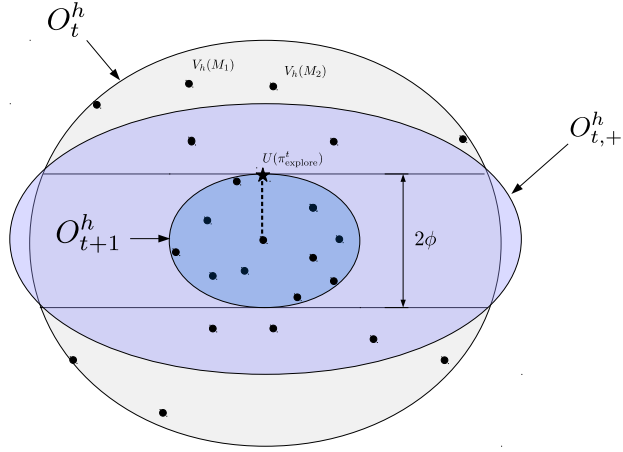}
  \caption{Illustration of geometric argument for $d=2$. Black dots represent embeddings of the models in $\mathcal{M}$, the star represents the embedding of the exploration policy $\pi_\mathrm{explore}^t$.} 
  \label{fig:ellipsoids}
\end{figure}

  We can then apply Lemma 11 in \cite{jiang2017}, (setting $B:= O_{t}^h, p:=U_{h_t}(\pi_\mathrm{explore}^t), v:=V_{h_t}(M'), \kappa:=6\sqrt{d}\phi$), and get:

  \begin{equation*}
    \frac{\mbox{vol}(O_{t+1}^{h_t})}{\mbox{vol}(O_{t}^{h_t})} \leq \frac{\mbox{vol}(O_{t, +}^{h_t})}{\mbox{vol}(O_{t}^{h_t})} \leq 3/5
  \end{equation*}

  This shows that if the algorithm does not terminate, then we shrink the volume of $O_t^{h_t}$ by a constant factor.
  To show that the number of iterations is small, we must now show that the initial volume is not too large and the final volume is not too small.
  Denote $\Phi := \sup_{\pi \in \Pi} \|U_{h_t}(\pi) \|_2$ and $\Psi := \sup_{M \in \mathcal{M}} \|V_{h_t}(M) \|_2$.
  For $O_1^h$, we have that $\mbox{vol}(O_1^h) \leq c_d \Psi^d$ where $c_d$ is the volume of the unit Euclidean ball in $d$ dimensions.
  For any $t$, we have

  \begin{equation*}
    O_t^h \supseteq \{q \in \mathbb{R}^d: \max_{p:\|p\|_2 \leq \Phi} q^\top p \leq 2\phi\} = \{q \in \mathbb{R}^d: \|q\|_2 \leq 2\phi/\Phi\}
  \end{equation*}

  Hence, at termination we must have that $\mbox{vol}(O_T^h) \geq c_d(2\phi/\Phi)^d$.
  Using the volume of $O_1^h$ and the lower bound of the volume of $O_T^h$ and the fact that every round we shrink the volume of $O_t^{h_t}$ by a constant factor, we must have that for any $h \in [H]$ the number of rounds for which $h_t=h$ is at most $d \log(\frac{\Phi \Psi}{2\phi}) / \log(5/3)$.
  Using the definition $\beta \geq \Phi \Psi$, this gives an iteration complexity of $H d \log(\frac{\beta}{2\phi})/\log(5/3)$.

\end{proof}

\begin{theorem}
  \label{sample-complexity}
  Assuming that $M^\star \in \mathcal{M}$, for any $\epsilon, \delta \in (0, 1]$ set $\phi=\frac{\epsilon}{24H^2|\mathcal{A}|^2\sqrt{d}}$ and denote $T=Hd \log(\frac{\beta}{2\phi}) / \log(5/3)$. Run Algorithm \ref{alg:main} with inputs $(\mathcal{M}, n, \phi)$ where $n=\Theta(H^4|\mathcal{A}|^4d \log(T|\mathcal{M}|/\delta)/\epsilon^2)$, and the \texttt{UpdateModelSet} routine given by Algorithm \ref{alg:update-version-space}. Then with probability at least $1-\delta$, Algorithm \ref{alg:main} outputs a policy $\pi_{\mathrm{exploit}}$ such that $v_{\pi_{\mathrm{exploit}}} \geq v^* - \epsilon$.
    The number of trajectories collected is at most $\tilde{O}\Big( \frac{H^5d^2|\mathcal{A}|^4}{\epsilon^2} \log\Big(\frac{T|\mathcal{M}||\Pi|}{\delta} \Big) \Big)$.

\end{theorem}

\begin{proof}
  We condition on the event that $|\widetilde{\mathcal{W}}(\pi_\mathrm{explore}^t, M, h) - \mathcal{W}(\pi_\mathrm{explore}^t, M, h)| \leq \phi$ for all $t$ and $h \in [H], M \in \mathcal{M}$.
  Under this condition, by Lemma \ref{iteration-complexity} we know that the algorithm must terminate in at most $Hd \log(\frac{\beta}{2\phi}) / \log(5/3)$ iterations.
  Once the algorithm terminates, we know that we must have an $\epsilon$-optimal policy by Lemma \ref{terminate-explore}.
  Now we show that this condition holds with probability at least $1-\delta$.
  Applying Technical Lemma \ref{misfit-deviation} and performing a union bound over all $h \in \{1,...,H\}$ and $t \in \{1,...,T\}$, we have that with probability at least $1-\delta$:

  \begin{align*}
    \Big| \widetilde{\mathcal{W}}(\pi_\mathrm{explore}^t, M, h) - \mathcal{W}(\pi_\mathrm{explore}^t, M, h)  \Big| &\leq \frac{4\log(4TH|\mathcal{M}||\Pi|/\delta)}{3n} + 4\sqrt{\frac{\log(4TH|\mathcal{M}||\Pi|/\delta)}{n}} \\
    &\leq 8\sqrt{\frac{\log(4TH|\mathcal{M}||\Pi|/\delta)}{n}} 
  \end{align*}

  for all $T$ iterations of the algorithm and $n > 4\log(4TH|\mathcal{M}||\Pi/\delta)/3$.
  Requiring this upper bound to be less than $\phi=\frac{\epsilon}{24H^2|\mathcal{A}|^2\sqrt{d}}$ and solving for $n$, we get:

  \begin{align*}
    8\sqrt{\frac{\log(4TH|\mathcal{M}||\Pi|/\delta)}{n}}  &\leq \frac{\epsilon}{24H^2|\mathcal{A}|^2\sqrt{d}} \\
    64\frac{\log(4TH|\mathcal{M}||\Pi|/\delta)}{n}  &\leq \frac{\epsilon^2}{576H^4|\mathcal{A}|^4d} \\
    \frac{36864H^4|\mathcal{A}|^4d\log(4TH|\mathcal{M}||\Pi|/\delta)}{\epsilon^2}  &\leq n \\
  \end{align*}

  Since we are sampling this number of trajectories at each iteration of the algorithm, the total number of trajectories is therefore $n\cdot T = \tilde{\mathcal{O}}(\frac{H^5d^2|\mathcal{A}|^4}{\epsilon^2}\log(\frac{T|\mathcal{M}||\Pi|}{\delta}))$.
\end{proof}

\subsection{Extension to Unknown $d$}
\label{doubling-trick}

\begin{algorithm}[h]
  \begin{algorithmic}[1]
    \FOR{$i=1, 2, ...$}
    \STATE Set $d_i \leftarrow 2^i$
    \STATE Set $\delta_i \leftarrow \frac{\delta}{i(i+1)}$
    \STATE Set $\phi_i \leftarrow \frac{\epsilon}{24H^2|\mathcal{A}|^2\sqrt{d_i}}$
    \STATE Set $n_i = \frac{36864H^4 |\mathcal{A}|^4 d_i \log(4TH|\mathcal{M}||\Pi|/\delta_i)}{\epsilon^2}$
    \STATE Run $\mathrm{DREEM}(\mathcal{M}, \Pi, n_i, \epsilon, \phi_i)$ until it returns a policy $\pi$ or $t > Hd_i \log(\frac{\beta}{2\phi_i}) / \log(5/3)$
    \IF{a policy $\pi$ was returned}
    \STATE Return $\pi$
    \ENDIF
    \ENDFOR
\end{algorithmic}
\caption{$(\mathcal{M}, \Pi, \epsilon, \delta)$}
\label{alg:doubling}
\end{algorithm}

Algorithm \ref{alg:doubling} shows how a near-optimal policy can be computed without requiring knowledge of the $d$ parameter.
It operates by running DREEM as a subroutine using guesses for $d$ which follow a doubling schedule with adjusted values of the $\delta$ parameter.

First note that since we assign $\frac{\delta}{i(i+1)}$ failure probability to each round of Algorithm \ref{alg:doubling}, the total probability that any of the subroutines returns a suboptimal policy is $\sum_{i=1}^\infty \frac{\delta}{i(i+1)} = \delta \sum_{i=1}^\infty (\frac{1}{i} - \frac{1}{i+1}) = \delta$.
Also note that $M^\star$ is never eliminated. Therefore with probability $1-\delta$, if the algorithm does return a policy, it is near optimal.
It remains to show that Algorithm \ref{alg:doubling} terminates. We know that the subroutine terminates with a near-optimal policy when we reach the first iteration $i$ such that $d \leq d_i=2^i$. Then we must have $d_{i-1} < d \leq d_i \implies 2 d_{i-1} < 2d \leq 2d_i \implies d_i \leq 2d \implies 2^i \leq 2d \implies i \leq \log_2d + 1$, so the algorithm terminates after $\log_2d + 1$ iterations. The sample complexity of each subroutine call is monotonically increasing, and the sample complexity of the last call is $\mathcal{\tilde{O}}(\frac{H^5d_i|\mathcal{A}|^4}{\epsilon^2}\log(\frac{T|\mathcal{M}|}{\delta_i})) = \mathcal{\tilde{O}}(\frac{H^52d|\mathcal{A}|^4}{\epsilon^2}\log(\frac{(\log d)^2T|\mathcal{M}|}{\delta})) = \mathcal{\tilde{O}}(\frac{H^5d|\mathcal{A}|^4}{\epsilon^2}\log(\frac{T|\mathcal{M}|}{\delta}))$, where we have suppressed constant factors and factors which are logarithmic in $d$ at the last step. Combining this with the fact that there are at most $\log_2d+1$ iterations, we see that the sample complexity of Algorithm \ref{alg:doubling} is the same as Algorithm \ref{alg:main} up to factors which are logarithmic in $d$.

\section{Practical Algorithm Details}

\subsection{Model Updates}


For the environments with deterministic dyanamics (Maze and Continuous Control), we found it helpful to train the models to make multi-step rather than single-step predictions.
For a trajectory $\tau=(s_i, a_i, s_{i+1}, a_{i+1}, ...s_{i+K+1}) \in \mathcal{R}$ and model $M$ with parameters $\theta$, the loss is given by:

\begin{align*}
  \mathcal{L}(\theta, \tau) = \sum_{j=1}^K \| s_{i+j+1} - M_\theta(\tilde{s}_{i+j}, a_{i+j}) \|_2^2 \mbox{ such that }
  \tilde{s}_{i+j} =
  \begin{cases}
    s_i &\mbox{ if } j = 0 \\
    M_\theta(\tilde{s}_{i+j-1}, a_{i+j-1}) &\mbox{ else }
  \end{cases}
\end{align*}

Beyond the first step, the model takes as input its prediction from the previous step, and gradients are backpropagated through the model unrolled over $K$ time steps.
This helps the models make more robust predictions over longer timescales, since errors which are magnified over time get penalized and the models are trained on noisy inputs.
We also used a simple form of prioritized experience replay \cite{PER}, where we sample trajectories from the last epoch with higher probability ($p=0.5$) and from all remaining epochs uniformly. This helps the models quickly learn from recent experience.
For the stochastic environment (combination lock), we found that single-step predictions worked well.

\subsection{Planning}
\label{neural-e3-planning}

\subsubsection{Deterministic Dynamics}
\label{planning-deterministic}

Algorithm \ref{alg:search} shows the procedure for searching in a continuous state space when the dynamics are deterministic (note the start state can still be stochastic).
If the state space is discrete, exponential time complexity can be avoided by marking states as visited and only expanding unvisited states, an idea which is used in breadth-first or depth-first search.
Here we generalize this idea for continuous spaces using a priority queue, where expanded states are assigned a priority based on their minimum distance to other states in the currently expanded search tree. If two action sequences lead to nearby states, only one of these states is likely to be expanded given a fixed computational budget as the other will be given low priority due to its proximity to the first. The algorithm returns variable length action sequences, and may be called multiple times within an episode.

\begin{algorithm}[h!]
\begin{algorithmic}[1]
  \STATE \textbf{Input} Set $\mathcal{M} = \{ f_{\theta_1}, ..., f_{\theta_E} \}$ of dynamics models, current state $s$, max graph size $N_{\mathrm{max}}$.
  \STATE Define root node: for $i=1,...,E$ set $v.s_i=s$
  \STATE Set $v.\hat{s} \leftarrow s$
  \STATE Set $v.priority \leftarrow \infty, v.\pi \leftarrow [] $
  \STATE Initialize graph $\mathcal{V} \leftarrow \{ v \}$
\WHILE{$|\mathcal{V}| < N_{\mathrm{max}}$}
\STATE Pick vertex to expand: $v \leftarrow \mbox{argmax}_{v \in \mathcal{V}} \Big [ v.priority \Big ]$
\STATE Set $v.priority \leftarrow -\infty$
\FOR{$a \in \mathcal{A}$}
\IF{\texttt{mode = explore}}
\STATE Utility is maximum disagreement: $u \leftarrow \max_{f_{\theta_i}, f_{\theta_j} \in \mathcal{M}}\|f_{\theta_i}(v.s_i, a) - f_{\theta_j}(v.s_j, a) \|_2^2$
\ELSIF{\texttt{mode = exploit}}
\STATE Utility is average predicted reward: $u \leftarrow \frac{1}{E}\sum_{i=1}^E R^\star(f_{\theta_i}(v.s_i, a))$
\ENDIF
\STATE Define new node $v'$ with $v'.\pi \leftarrow append(v.\pi, a)$
\STATE For $i=1,...,E$, set $v'.s_i \leftarrow f_{\theta_i}(v.s_i, a)$
\STATE Set $v'.\hat{s} \leftarrow \frac{1}{E} \sum_{i=1}^E v'.s_i$
\STATE Set $v'.priority \leftarrow \min_{v \in \mathcal{V}} \| v'.\hat{s} - v.\hat{s} \|_2$
\STATE Set $v'.utility \leftarrow v.utility + u$
\STATE $\mathcal{V} \leftarrow \mathcal{V} \cup \{v'\}$
\ENDFOR
\ENDWHILE
\STATE $v^{\star} \leftarrow \mbox{argmax}_{v \in \mathcal{V}} v.utility / |v.\pi|$
\STATE Return $v^{\star}.\pi$
\end{algorithmic}
\caption{$\texttt{DeterministicPlanner}(s, \mathcal{M}, N_{\mathrm{max}}, \texttt{mode})$}
\label{alg:search}
\end{algorithm}

\subsubsection{Stochastic Dynamics}
\label{planning-stochastic}

When planning in a stochastic environment, we use Monte-Carlo Tree Search where a given node $\nu$ in the tree at depth $h$ corresponding to a fixed action sequence $\pi_A$ (of length $h$) consists of empirical estimates of $P_{M_1}^{\pi_A, h}, ..., P_{M_E}^{\pi_A, h}$ for each model in the ensemble. Concretely, $\nu$ is represented as a tensor of size $|E| \times K \times m$ where $|E|$ is the number of models in the ensemble, $K$ is the number of samples drawn from each model $M$ to estimate its predicted distribution $P_M^{\pi_A, h}$, and $m$ is the dimension of the state vector.
The root node is initialized with the current state $s$, i.e. $S^\mathrm{root}_{e, k} = s$ for all $1 \leq e \leq E, 1 \leq k \leq K$.
Given an action $a \in \mathcal{A}$ applied at a node $\nu$, the next node is computed as follows: $\nu'_{e, k} \sim M_e(\nu_{e, k}, a)$.

The rewards at each node, which are then used to choose which action to execute in the real environment, depend on whether the algorithm is in explore or exploit mode.
In explore mode, the reward is given by $R_\mathrm{explore}(\nu) = \max_{1 \leq e, e' \leq E} \| \hat{P}_{M_e}(\cdot) - \hat{P}_{M_{e'}}(\cdot)  \|_{TV}$, where $\hat{P}_{M_e}(\cdot)$ is the empirical distribution computed using the $K$ samples $\nu_{e, :}$ drawn from model $M_e$.
In exploit mode, the reward is given by $R_\mathrm{exploit}(\nu) = \frac{1}{E \cdot K} \sum_{e=1}^E \sum_{k=1}^K R^\star(\nu_{e, k})$, i.e. the mean reward across all samples and all models in the ensemble.
After a fixed number of playouts, the MCTS procedure returns a sequence of actions which maximizes the expected exploration or exploitation reward. We execute the first action in this sequence, and then replan at every step. See our code release for full details.

\section{Experiment Details}
\label{appendix-experiments}

For Neural-E$^3$, we found that specifying a number of exploration epochs was simpler than tuning the $\epsilon$ parameter in Algorithm 1, which determines when to switch to the exploit phase and which is task-dependent. This is listed in the table of hyperparameters.

\subsection{Stochastic Combination Lock}
\label{appendix:combolock}

The environment consists of $H$ levels with $3$ underlying states per level (denoted $s_{1, h}, s_{2, h}, s_{3, h}$) and 4 possible actions. The states $s_{3, :}$ are dead states from which it is impossible to recover: all actions from $s_{3, h}$ lead to $s_{3, h+1}$ with probability 1.
2 actions lead from each of the states $s_{1, h}$ and $s_{2, h}$ to the dead state $s_{3, h+1}$, and the other two actions lead to one of $s_{1, h+1}$ and $s_{2, h+1}$. Which action leads to which state is randomly determined when the environment is initialized and kept fixed thereafter. This means that simply repeating a single action is unlikely to lead to the reward.
With probability $\alpha=0.1$, the effect of the actions leading to the good states at the next level is flipped.
Therefore, executing a preplanned action sequence without accounting for intermediate observations is likely to lead to the dead states.

\textbf{Standard Reward Variant:} The reward is zero everywhere except at the last two states, where a reward of 5 is given for one of the actions. 

\textbf{Antishaped Reward Variant:} As above, a reward of 5 is given at the last two states for one of the actions. Furthermore, a reward of 0.1 is given for transitioning to any of the dead states (for example, from $s_{1, h}$ to $s_{3, h+1}$), and a negative reward of $-1/H$ is given for transitioning to any state which is not a dead state (for example, from $s_{1, h}$ to $s_{1, h+1}$). This means that until the agent has explored the last states which give high reward, the locally optimal policy appears to be to transition to the dead states as quickly as possible.

\begin{figure}[ht]
  \centering
\subfigure[With standard rewards]{\includegraphics[width=0.24\textwidth]{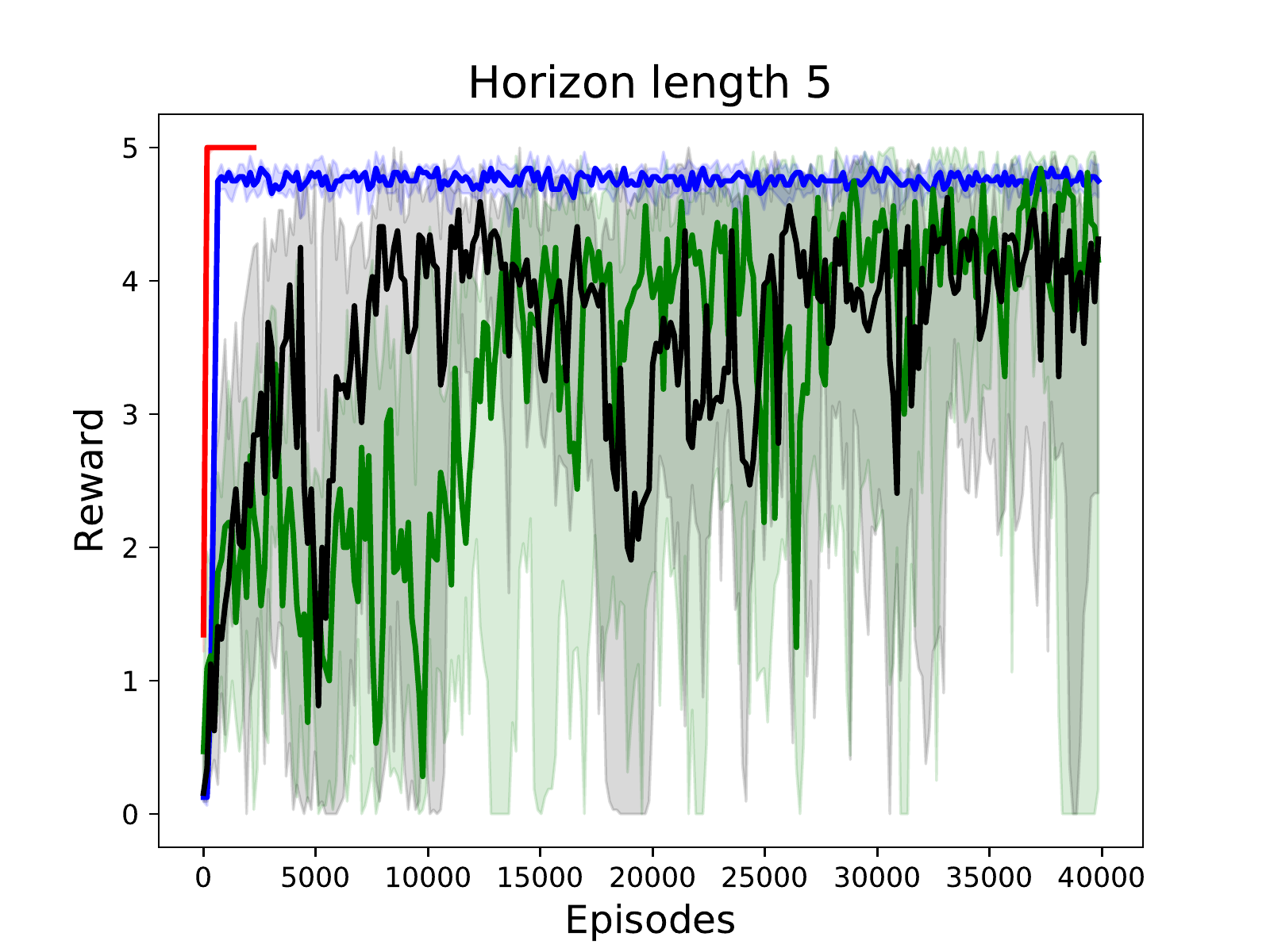}
\includegraphics[width=0.24\textwidth]{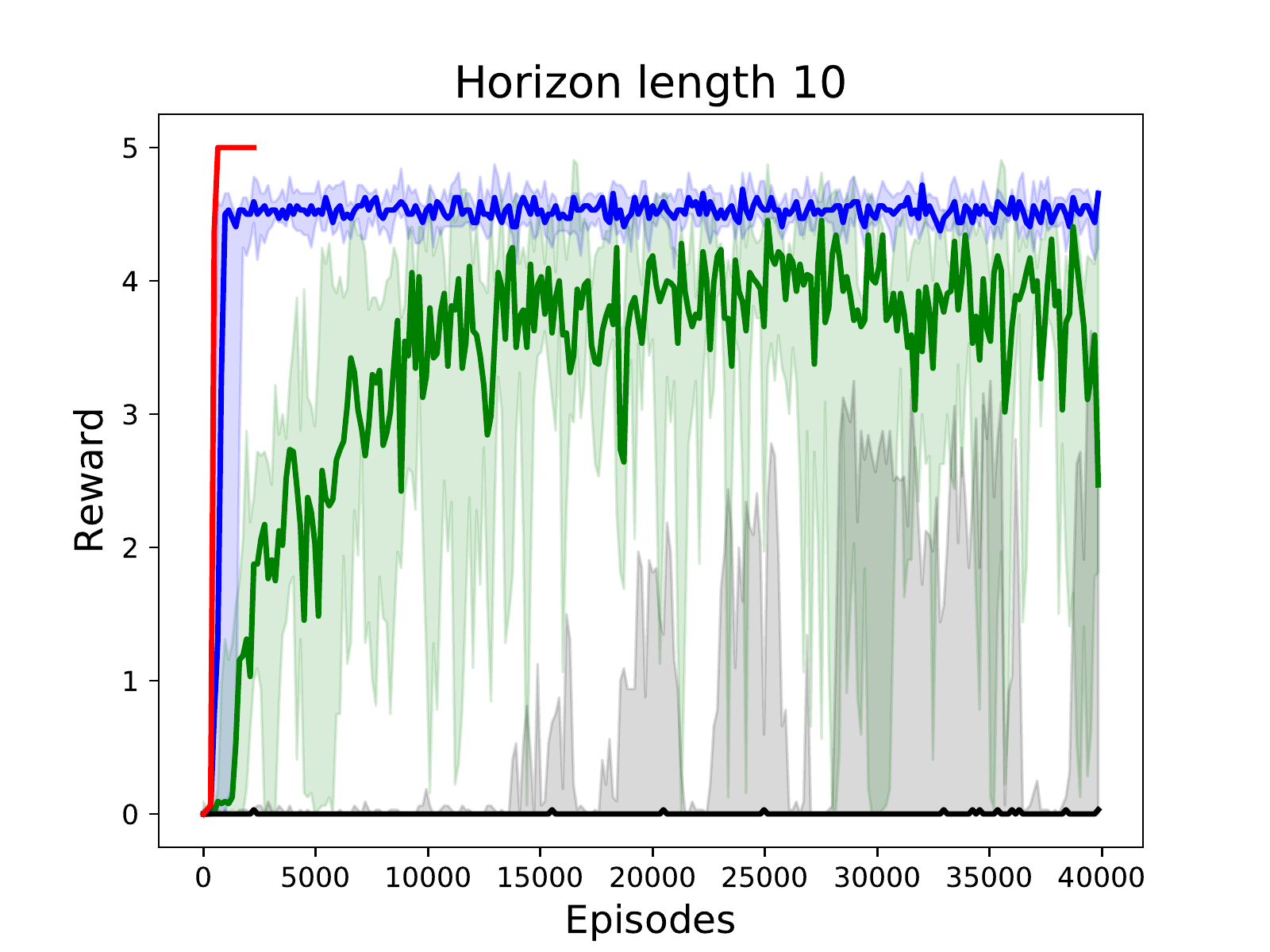}
\includegraphics[width=0.24\textwidth]{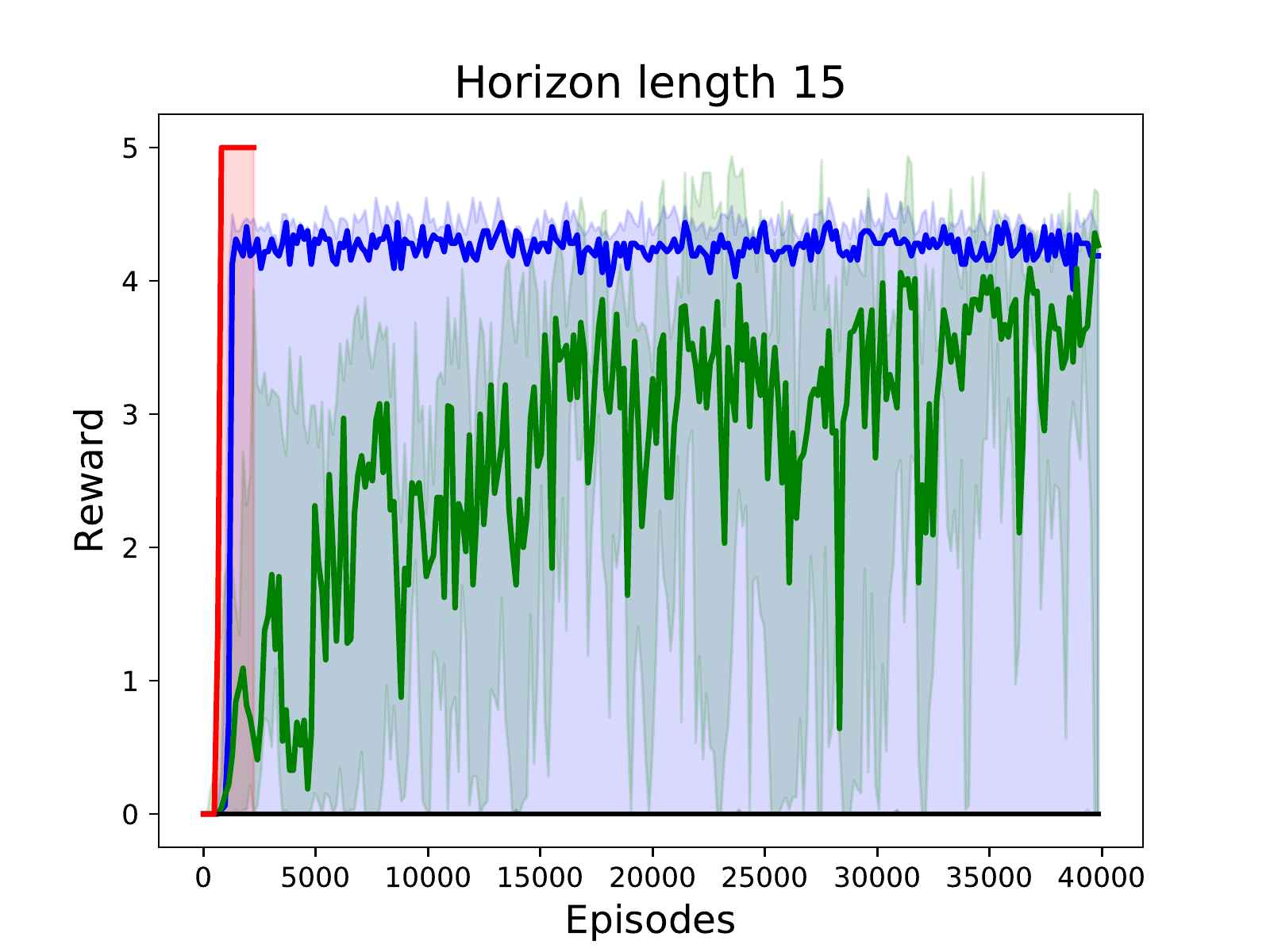}
\includegraphics[width=0.24\textwidth]{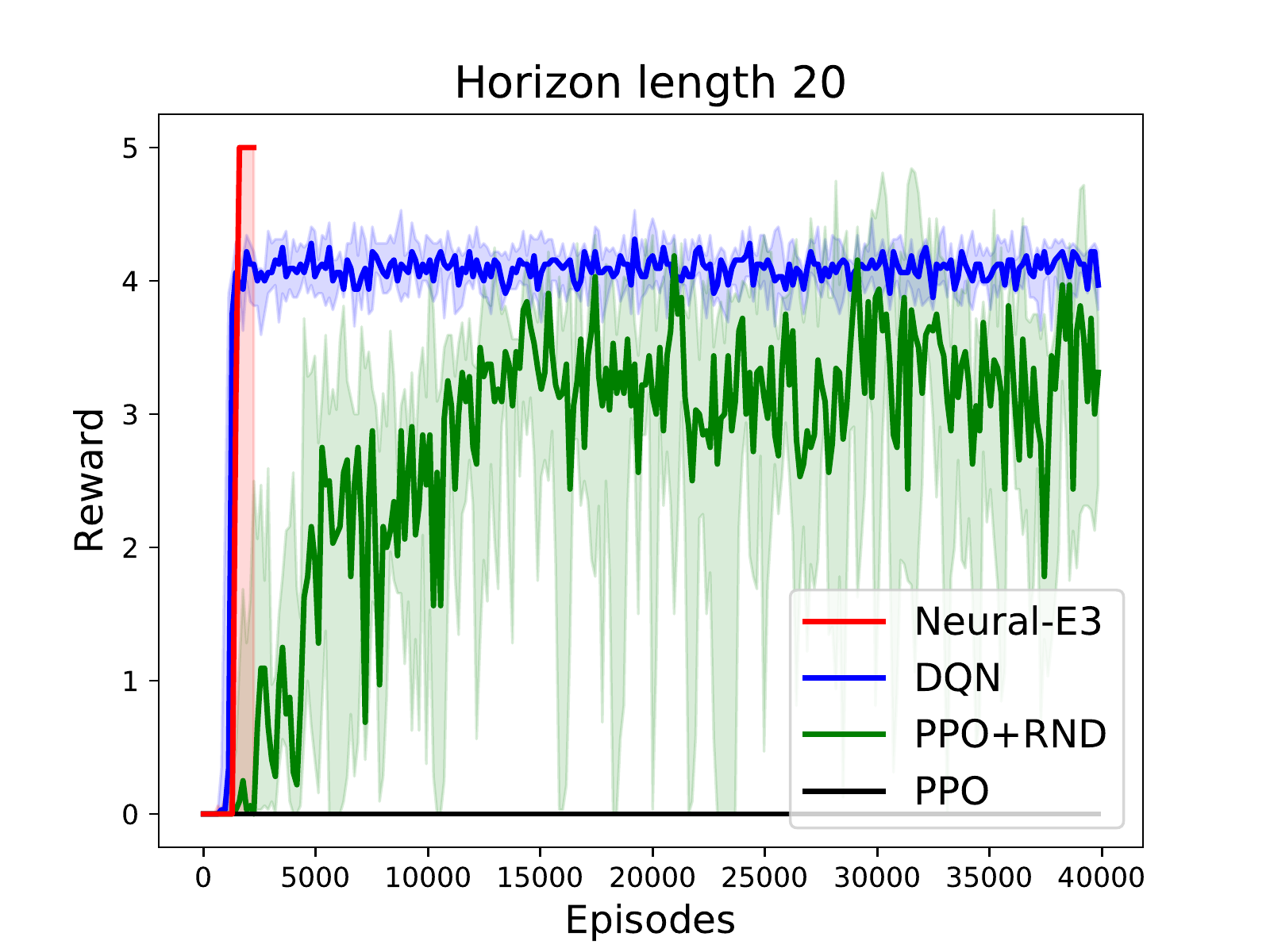}}
  \centering
\subfigure[With antishaped rewards]{\includegraphics[width=0.24\textwidth]{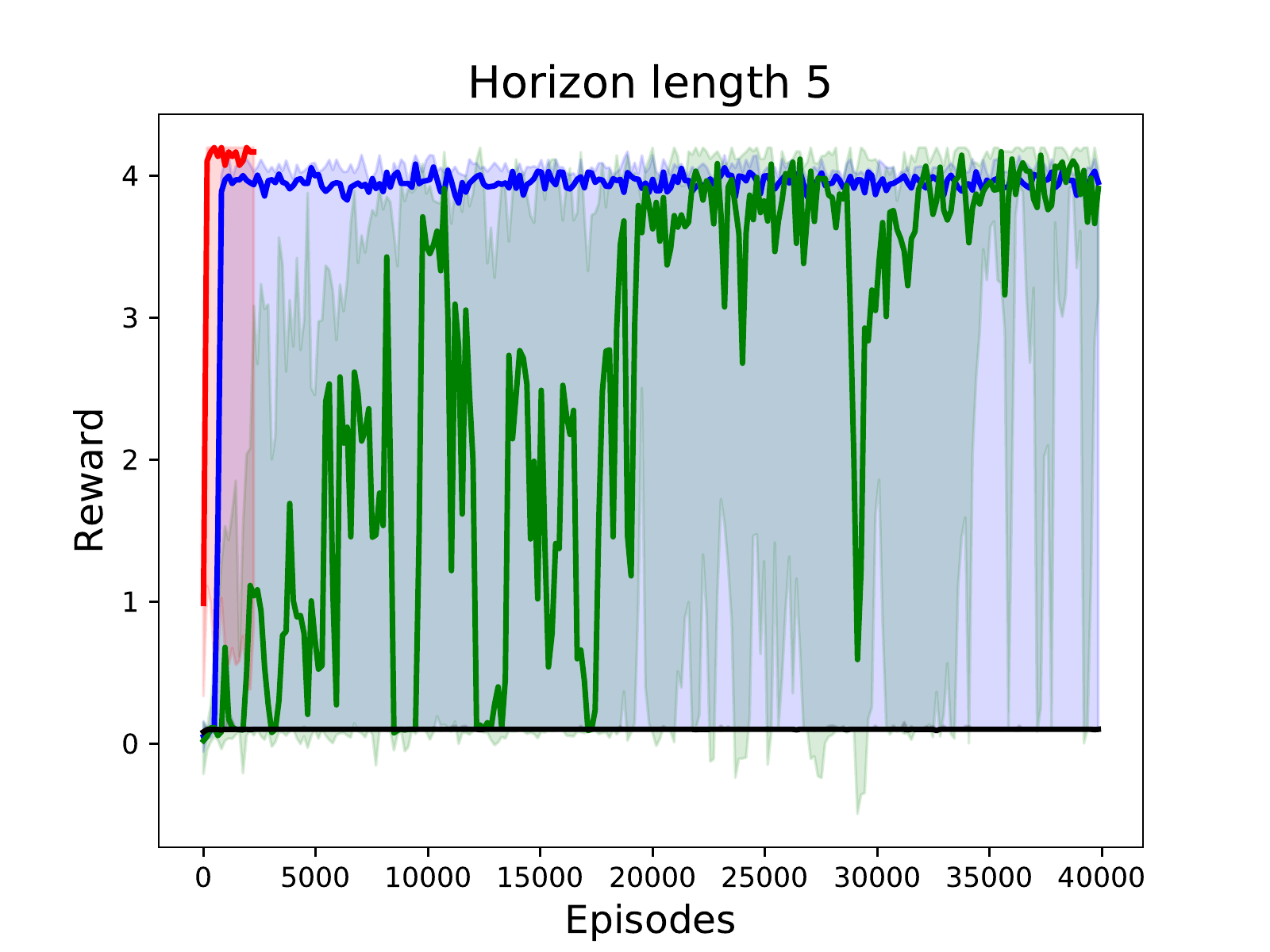}
\includegraphics[width=0.24\textwidth]{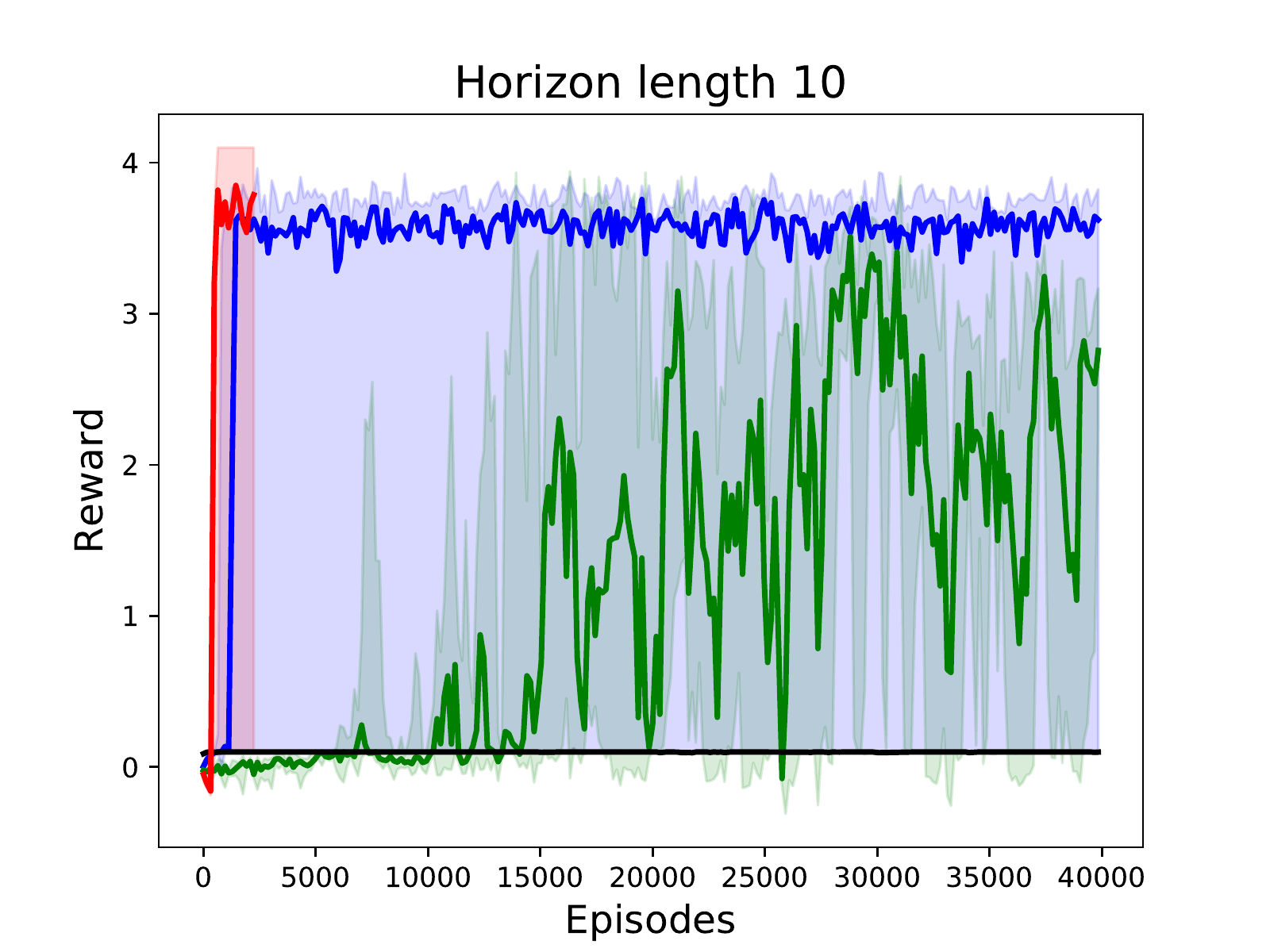}
\includegraphics[width=0.24\textwidth]{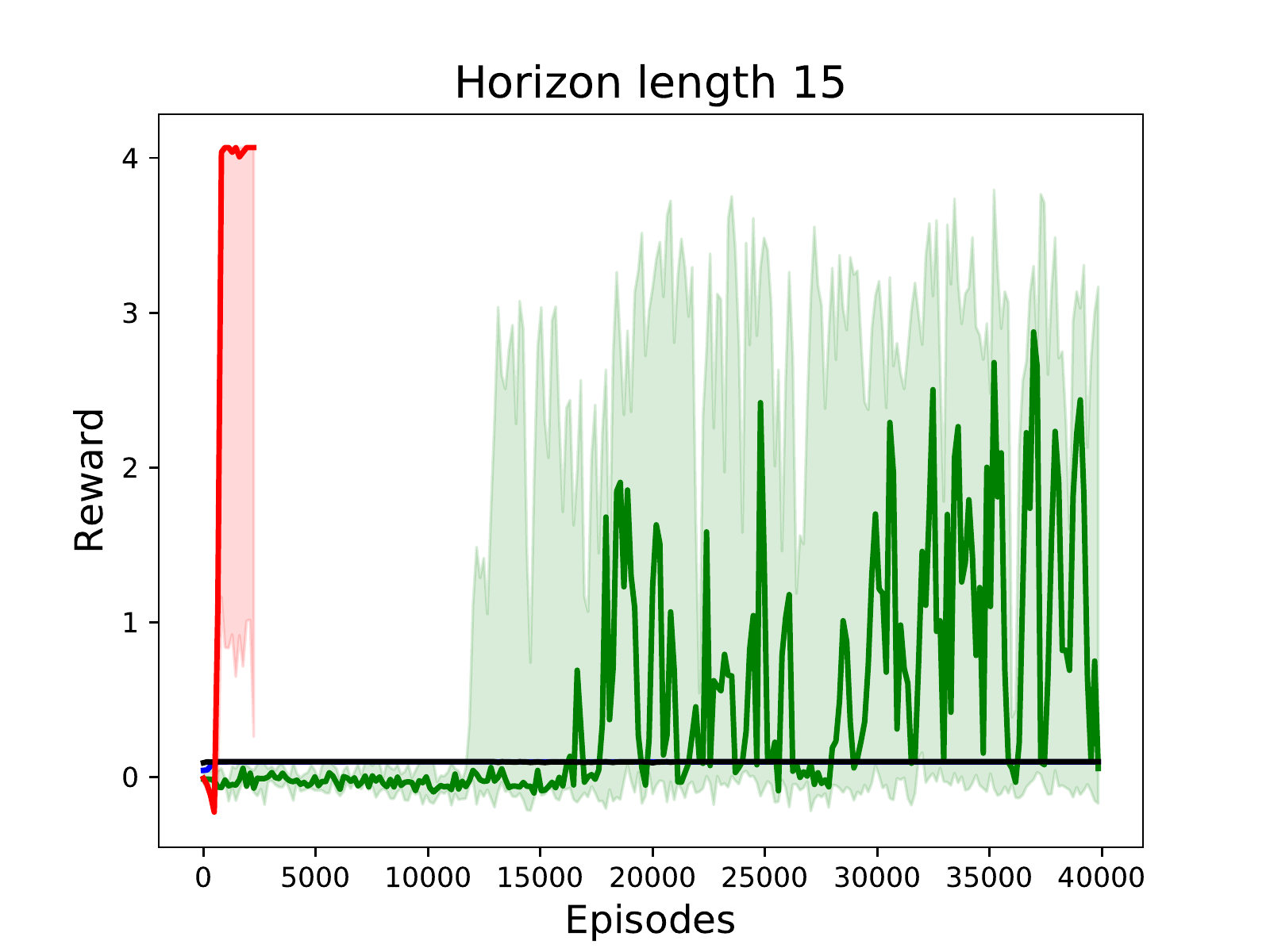}
\includegraphics[width=0.24\textwidth]{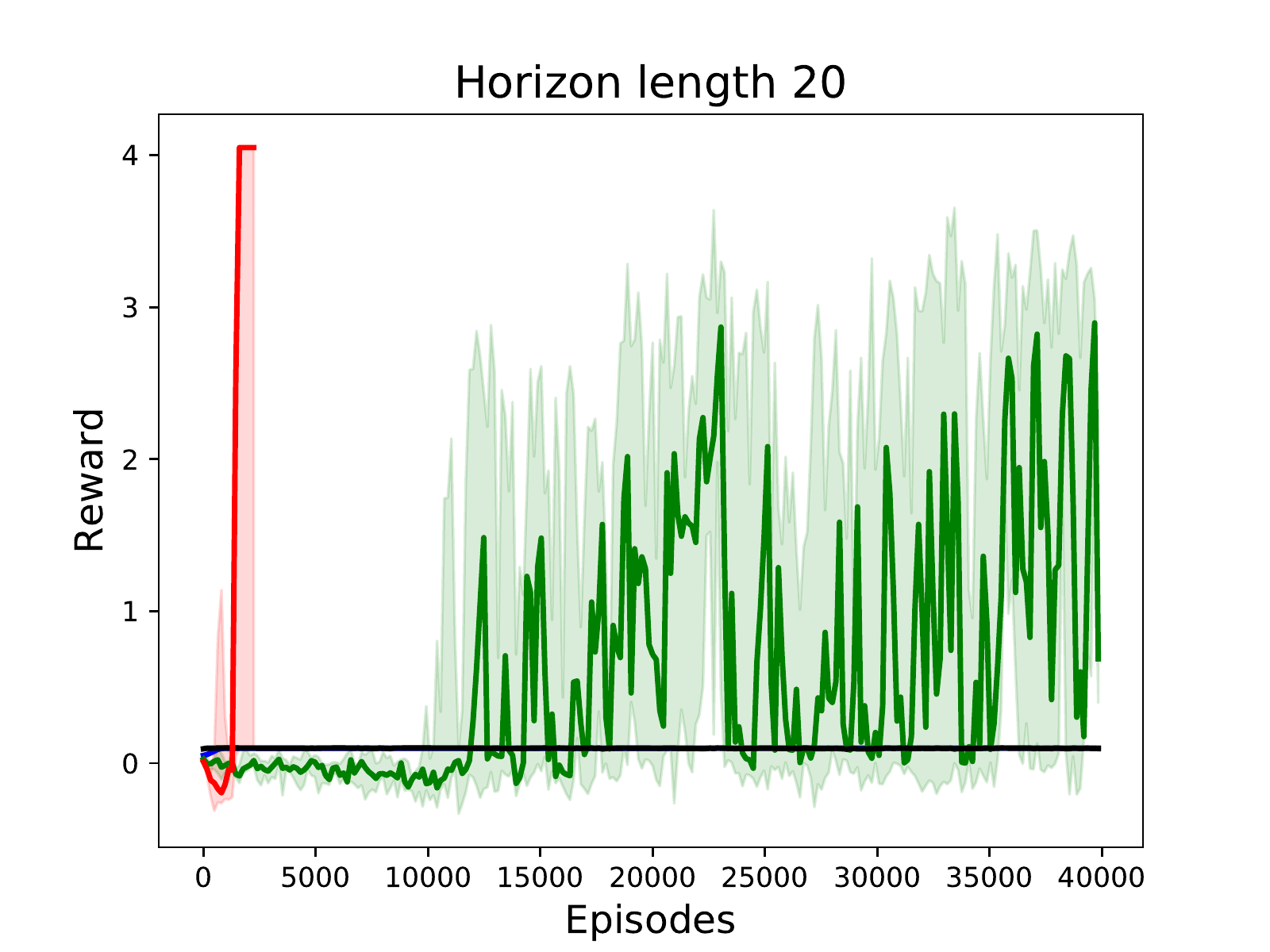}}
\caption{Results on the Stochastic Combination Lock task given more episodes. PPO+RND is able to eventually achieve reasonable performance given enough episodes.}
\label{fig:extra-combolock-results}
\end{figure}

\begin{table}[h!]
  \caption{DQN Hyperparameters}
  \centering
  \begin{tabular}{lll}
    \toprule
    Hyperparameter & Values Considered & Final Value \\
    \hline
    Learning Rate & $0.01, 0.001, 0.0001$ & $0.01$ \\
    Hidden Layer Size     & $64$  & $64$ \\
    Prioritized Replay & \texttt{true}& \texttt{true} \\
    Discount Factor & $0.99$ & $0.99$ \\
    Exploration Fraction (episodes) & $\{0.1, 0.01, 0.001\}$ & $0.001$ for standard rewards \\
    & & $0.01$ for antishaped rewards \\
    \bottomrule
  \end{tabular}
  \label{table:DQN-lock-hyperparams}
\end{table}

Figure \ref{fig:extra-combolock-results} shows results for both variants of the task for larger numbers of episodes. A somewhat surprising result was that for the standard variant of the task, the DQN is still able to achieve good performance for longer horizons, using much fewer samples than PPO+RND. We found that the DQN performed best when the exploration fraction is set to be very low ($0.001$ as shown in Table \ref{table:DQN-lock-hyperparams}), meaning that the DQN agent quickly begins to act greedily. This suggests that acting greedily leads the agent to explore the environment better than uniform exploration. Uniform exploration leads to a vanishingly small chance of reaching the reward ($\approx 10^{-6}$ for $H=20$). One explanation could be that the network happens to be initialized in a manner that gives optimistic estimates for the Q-values. We found that the DQN performance was highly dependent on implementation details, for example, the implementation in \cite{deeprl} gave very poor results, as did removing the prioritized experience replay. 

\begin{table}[h!]
  \caption{PPO+RND Hyperparameters}
  \centering
  \begin{tabular}{lll}
    \toprule
    Hyperparameter & Values Considered & Final Value \\
    \hline
    Learning Rate & $0.01, 0.001, 0.0001$ & $0.001$ \\
    Hidden Layer Size     & $64$  & $64$ \\
    $\gamma_I$     & $0.99$  & $0.99$ \\
    $\gamma_E$      & $0.999$  & $0.999$ \\
    $\lambda$ & 0.95 & 0.95 \\
    Intrinsic Reward coefficient     & $1.0$  & $1.0$ \\
    Extrinsic Reward coefficient     & $2, 100$  & $100$ \\
    \bottomrule
  \end{tabular}
\end{table}

\begin{table}[h!]
  \caption{$E^3$ Hyperparameters}
  \centering
  \begin{tabular}{lll}
    \toprule
    Hyperparameter & Values Considered & Final Value \\
    \hline
    Learning Rate & $0.01, 0.001$ & $0.01$ \\
    Hidden Layer Size     & $50, 100$  & $50$ \\
    Ensemble Size     & $5, 10$  & $5$ \\
    Minibatch Size     & $100$  & $100$ \\
    Number of Exploration Epochs     & $25H, 50H, 75H$  & Horizon-dependent: \\
    & & $H=5: 25H, H=10: 50H$ \\
    & & $H=15: 50H, H=20: 75H$\\
    Exploration Episodes per Epoch     & $1$  & $1$ \\
    Model Updates per Epoch     & $100$  & $100$ \\
    MCTS playouts    & $200$  & $200$ \\
    MCTS samples ($K$)    & $100$  & $100$ \\
    \bottomrule
  \end{tabular}
\end{table}

For Neural-E$^3$, we found that training a DQN offline using the data collected in the replay buffer (as described in Section 4.3) performed better than using MCTS to maximize the reward, especially on the task variant with antishaped rewards.
This is likely because MCTS biases the search tree towards action sequences which accumulate the best reward so far, and so the misleading rewards can lead the search procedure away from action sequences which produce the globally optimal reward. All the Neural-E$^3$ results reported use the DQN exploitation method.

\subsection{Maze Domain}
\label{appendix-maze}

We used the source code for the maze environment provided by the authors \url{https://github.com/junhyukoh/value-prediction-network}, and set the number of goals to 1 and the time limit to 100.
All results are reported using 3 random seeds.

The forward dynamics model architecture is a 3-layer convolutional network (1 convolutional layer followed by 2 deconvolutional layers, all with 16 feature maps).
Actions are embedded to a 16-dimensional vector replicated across all spatial locations and added to the feature maps.
A separate reward head consists of 2 strided convolution layers followed by a fully-connected layer producing a scalar.

\begin{table}[h!]
  \caption{DQN Hyperparameters}
  \centering
  \begin{tabular}{lll}
    \toprule
    Hyperparameter & Values Considered & Final Value \\
    \hline
    Learning Rate & $10^{-3}, 10^{-4}, 10^{-5}$ & $10^{-4}$ \\
    Feature Maps     & $8, 32$  & $8$ \\
    Convolutional Layers     & $1, 2, 3$  & $1$ \\
    Hidden Layer Size     & $64, 256$  & $64$ \\
    Prioritized Replay & \texttt{true}& \texttt{true} \\
    Parameter Noise & \texttt{false} & \texttt{false} \\
    Discount Factor & $0.99$ & $0.99$ \\
    \bottomrule
  \end{tabular}
\end{table}

\begin{table}[h!]
  \caption{$E^3$ Hyperparameters}
  \centering
  \begin{tabular}{lll}
    \toprule
    Hyperparameter & Values Considered & Final Value \\
    \hline
    Learning Rate & $10^{-3}, 10^{-4}$ & $10^{-3}$ \\
    Number of Feature Maps     & $16$  & $16$ \\
    Hidden Layer Size     & $64$  & $64$ \\
    Ensemble Size     & $4, 8$  & $4$ \\
    Minibatch Size     & $64$  & $64$ \\
    Number of exploration Epochs     & $5, 10$  & $5$ \\
    Exploration Episodes per Epoch     & $10$  & $10$ \\
    Model Updates per Epoch     & $10000$  & $10000$ \\
    Unrolling steps ($K$)     & $10, 20$  & $10$ \\
    Maximum Graph Size during Planning ($N_{\mathrm{max}}$)     & $500, 1000, 2000, 5000$  & $2000$ \\
    \bottomrule
  \end{tabular}
\end{table}

\begin{figure}
  \centering
  \includegraphics[width=0.7\textwidth]{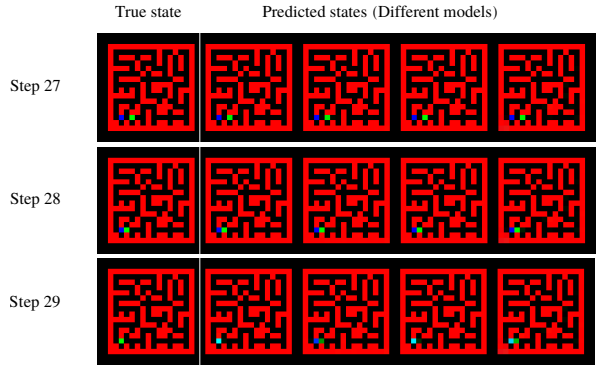}
  \caption{Predictions by the different dynamics models in the ensemble for the Maze task, 29 steps into the future (best viewed in color). The green dot is the agent and the blue dot is the goal. The models all agree in their predictions up to steps 27 and 28, but disagree for step 29 where the agent collects the reward.}
  \label{fig:predictions-crop}
\end{figure}

\subsection{Continuous Control Domains}

We used the environments provided by OpenAI Gym \cite{Gym}, available at: \url{https://gym.openai.com/envs/#classic_control}.
In initial experiments we experimented with adding parameter noise to the DQN, but found that this did not help. 

\begin{table}[h!]
  \caption{DQN Hyperparameters}
  \centering
  \begin{tabular}{lll}
    \toprule
    Hyperparameter & Values Considered & Final Value \\
    \hline
    Learning Rate & $10^{-2}, 10^{-3}, 10^{-4}$ & $10^{-3}$ \\
    Hidden Layer Size     & $64, 256$  & $64$ \\
    Prioritized Replay & \texttt{true}, \texttt{false} & \texttt{true} \\
    Parameter Noise & \texttt{true}, \texttt{false} & \texttt{false} \\
    Discount Factor & $0.99$ & $0.99$ \\
    \bottomrule
  \end{tabular}
\end{table}

\begin{table}[h!]
  \caption{PPO+RND Hyperparameters}
  \centering
  \begin{tabular}{lll}
    \toprule
    Hyperparameter & Values Considered & Final Value \\
    \hline
    Learning Rate & $10^{-3}, 10^{-4}, 10^{-5}$ & $10^{-4}$ \\
    Hidden Layer Size     & $64$  & $64$ \\
    $\gamma_I$     & $0.99$  & $0.99$ \\
    $\gamma_E$      & $0.999$  & $0.999$ \\
    $\lambda$ & 0.95 & 0.95 \\
    Intrinsic Reward coefficient     & $1.0$  & $1.0$ \\
    Extrinsic Reward coefficient     & $2$  & $2$ \\
    \bottomrule
  \end{tabular}
\end{table}

The forward model architecture is a 3-layer MLP with LeakyReLU non-linearities. The action is embedded to a vector of size 64 and multiplied component-wise with the first layer of hidden units. All models are trained using Adam \cite{ADAM}.

\begin{table}[h!]
  \caption{$E^3$ Hyperparameters}
  \centering
  \begin{tabular}{lll}
    \toprule
    Hyperparameter & Values Considered & Final Value \\
    \hline
    Learning Rate & $10^{-3}, 10^{-4}$ & $10^{-4}$ \\
    Hidden Layer Size     & $64$  & $64$ \\
    Ensemble Size     & $8$  & $8$ \\
    Minibatch Size     & $64$  & $64$ \\
    Number of exploration Epochs     & $10$  & $10$ \\
    Exploration Episodes per Epoch     & $\{10, 20\}$  & $10$ \\
    Model Updates per Epoch     & $2000$  & $2000$ \\
    Unrolling steps ($K$)     & $20$  & $20$ \\
    Maximum Graph Size during Planning ($N_{\mathrm{max}}$)     & $2000$  & $2000$ \\
    DQN Learning Rate & $10^{-2}, 3\cdot 10^{-3}, 1\cdot 10^{-3}, 3\cdot 10^{-4}, 1 \cdot 10^{-4}$ & $3\cdot 10^{-4}$ \\
    DQN Updates for Exploit Phase & $\{500000, 750000, 1000000\}$ & $750000$ \\
    DQN Target Network Update Frequency & $5000$ & $5000$ \\
    \bottomrule
  \end{tabular}
\end{table}

For the exploit phase, we initially train a DQN for 750000 updates on the data collected from the replay buffer.
It is then continued to be trained, and if performance begins decreasing, the model is reverted to its best performing set of weights and training is stopped.

\end{document}